\documentclass{article}

\usepackage{microtype}
\usepackage{graphicx}
\usepackage{subfigure}
\usepackage{multirow, makecell}
\usepackage{booktabs} 
\usepackage{hyperref}

\usepackage[accepted]{icml2021}

\icmltitlerunning{Binary Classification from Multiple Unlabeled Datasets via Surrogate Set Classification}

\usepackage{mathtools}
\usepackage{amsmath,amsthm,amssymb,bm}
\usepackage{mathrsfs}
\usepackage{nccbbb}
\newtheorem{theorem}{Theorem}

\newtheorem{lemma}[theorem]{Lemma}

\DeclareMathOperator*{\argmax}{argmax}
\DeclareMathOperator*{\argmin}{argmin}
\usepackage{algorithm,algorithmic}
\usepackage{enumitem}
\setlist{nosep}
\usepackage{booktabs}
\usepackage{multirow}
\usepackage{multicol}

\DeclareMathOperator{\sign}{\mathrm{sign}}

\newcommand{\pr}{\mathrm{Pr}}

\newcommand{\cG}{\mathcal{G}}

\begin{document}
\twocolumn[
\icmltitle{Binary Classification from Multiple Unlabeled Datasets\\ via Surrogate
Set Classification}



\icmlsetsymbol{equal}{*}

\begin{icmlauthorlist}
\icmlauthor{Nan Lu}{equal,to,goo}
\icmlauthor{Shida Lei}{equal,to}
\icmlauthor{Gang Niu}{goo}
\icmlauthor{Issei Sato}{to,goo}
\icmlauthor{Masashi Sugiyama}{goo,to}
\end{icmlauthorlist}

\icmlaffiliation{to}{The University of Tokyo, Tokyo, Japan}
\icmlaffiliation{goo}{RIKEN, Tokyo, Japan}

\icmlcorrespondingauthor{Nan Lu}{lu@ms.k.u-tokyo.ac.jp}
\icmlcorrespondingauthor{Shida Lei}{leishida@is.s.u-tokyo.ac.jp}

\icmlkeywords{Machine Learning, ICML}

\vskip 0.3in
]



\printAffiliationsAndNotice{\icmlEqualContribution} 

\begin{abstract}
To cope with high annotation costs, training a classifier only from \emph{weakly supervised data} has attracted a great deal of attention these days. 
Among various approaches, strengthening supervision from completely unsupervised classification is a promising direction, which typically employs \emph{class priors} as the only supervision and trains a
binary classifier from \emph{unlabeled}~(U) datasets. 
While existing \emph{risk-consistent} methods are theoretically grounded with high flexibility, they can learn only from \emph{two} U sets.
In this paper, we propose a new approach for binary classification from $m$ U sets for $m\ge2$.
Our key idea is to consider an auxiliary classification task called \emph{surrogate set classification}~(SSC), which is aimed at predicting from which U set each observed sample is drawn.
SSC can be solved by a standard (multi-class) classification method, and we use the SSC solution to obtain the final binary classifier through a certain linear-fractional transformation. 
We built our method in a flexible and efficient \emph{end-to-end} deep learning framework and prove it to be \emph{classifier-consistent}. 
Through experiments, we demonstrate the superiority of our proposed method over state-of-the-art methods.
\end{abstract}

\section{Introduction}
Deep learning with large-scale supervised training data has shown great success on various tasks \cite{goodfellow16DL}.  
However, in practice, obtaining \emph{strong supervision}, e.g., the complete ground-truth labels, for big data is very costly due to the expensive and time-consuming manual annotations \cite{zhou2018brief}.
Thus, it is desirable for machine learning techniques to work with \emph{weaker forms of supervision}, such as noisy labels \cite{natarajan13nips, patrini17cvpr, van2017theory, han2018co, han2020sigua, fang2020rethinking, xia2020part}, partial labels \cite{cour2011learning, ishida2017learning, ishida19icml, feng2020provably, lv2020progressive}, and pairwise comparison information \cite{bao18icml, xu2019uncoupled, feng2020pointwise}.

This paper focuses on a challenging setting which we call \emph{U$^m$ classification}: the goal is to learn a binary classifier from $m$ ($m\ge2$) sets of U data with different \emph{class priors}, i.e., the proportion of positives in each U set.
Such a learning scheme can be conceivable in many real-world scenarios.
For example, U sets with different class priors can be naturally collected from spatial or temporal differences.
Considering morbidity rates, they can be potential patient data collected from different areas \citep{croft2018urban}. 
Likewise, considering approval rates, they can be unlabeled voter data collected in different years \citep{newman2003integrity}. 
In such cases, individual labels are often not available due to privacy reasons, but the corresponding class priors of U sets, i.e., the morbidity rates or approval rates in the aforementioned examples, can be obtained from related medical reports or pre-existing census \cite{quadrianto09jmlr,ardehaly2017mining,tokunaga2020negative}, and is the unique weak supervision that will be leveraged in this work.

Breakthroughs in U$^m$ classification research were brought by \citet{menon15icml} and \citet{lu2018minimal} in proposing the \emph{risk-consistent} methods given two U sets.
Recently, \citet{scott2020learning} extended them to incorporate multiple U sets by two steps: firstly, pair all the U sets so that they are sufficiently different in each pair; secondly, linearly combine the unbiased balanced risk estimators obtained from each pair.
Although this method is advantageous since it is compatible with any model and stochastic optimizer, and is statistically consistent, there are several issues that may limit its potential for practical use:
first, the computational complexity for the optimal pairing strategy is $O(m^3)$ for $m$ U sets \cite{edmonds1972theoretical}, which cannot work efficiently with a large number of U sets;
second, the optimal combination weights are proved with strong model assumptions and thus remaining difficult to be tuned in practice.

Now, a natural question arises: can we propose a computationally efficient method for U$^m$ classification with both \emph{flexibility} on the choice of models and optimizers and \emph{theoretical guarantees}?
The answer is affirmative.

In this paper, we provide a new approach for U$^m$ classification by solving a Surrogate Set Classification task~(U$^m$-SSC).
More specifically, we regard the index of each U set as a \emph{surrogate-set label} and consider the supervised multi-class classification task of predicting the surrogate-set labels given observations.
The difficulty is how to link our desired binary classifier with the learned surrogate multi-class classifier.
To solve it, we theoretically bridge the original and surrogate class-posterior probabilities with a linear-fractional transformation, and then implement it by adding a transition layer to the neural network so that the trained model is guaranteed to be a good approximation of the original class-posterior probability.
Our proposed U$^m$-SSC scheme is built within an end-to-end framework, which is computationally efficient, compatible with any model architecture and stochastic optimization, and naturally incorporates multiple U sets.
Our contributions can be summarized as follows:
\begin{itemize}
\item Theoretically, we prove that the proposed U$^m$-SSC method is \emph{classifier-consistent} \cite{patrini17cvpr,lv2020progressive}, i.e., the classifier learned by solving the surrogate set classification task from multiple sets of U data converges to the optimal classifier learned from fully supervised data under mild conditions.
Then we establish an \emph{estimation error bound} of our method.
\item Practically, we propose an easy-to-implement, flexible, and computationally efficient method for U$^m$ classification, which is shown to outperform the state-of-the-art methods in experiments.
We also verify the robustness of the proposed method by simulating U$^m$ classification in the wild, 
e.g., on varied set sizes, set numbers, noisy class priors, 
and the results are promising.
\end{itemize}
Our method provides new perspectives of solving the U$^m$ classification problem, and is more suitable to be applied in practice given its theoretical and practical advantages.
\vspace{-0.5em}%

\section{Problem Setup and Related Work}
\label{sec:pre}%
In this section, we introduce some notations, formulate the U$^m$ classification problem, and review the related work.

\begin{table*}[t]
\caption{\centering Comparisons of the proposed method with previous works in the U$^m$ classification setting.}
\label{methods}
\vspace{1ex}%
\begin{center}
\newcommand{\tabincell}[2]{\begin{tabular}{@{}#1@{}}#2\end{tabular}}
\begin{tabular}{c|c c c c c}
\hline
Methods & \tabincell{c}{Deal with\\ 2+ sets} & \tabincell{c}{Theoretical \\ guarantee} & \tabincell{c}{No negative \\ training risk} & \tabincell{c}{Pre-computing \\ complexity} & \tabincell{c}{Risk \\ Measure}\\
\hline
$\widehat{R}_{\rm{U^2}}(f)$ \cite{lu2018minimal} & $\times$ & $\checkmark$ & $\times$ & $O(1)$ & Classification risk \eqref{Rpn}\\
$\widehat{R}_{\rm{U^2}\text{-}\rm{b}}(f)$ \cite{menon15icml} & $\times$ & $\checkmark$ & $\times$ & $O(1)$ & Balanced risk \eqref{balrisk}\\      
$\widehat{R}_{\rm{U^2}\text{-}\rm{c}}(f)$ \cite{lu2020mitigating} & $\times$ & $\checkmark$ & $\checkmark$ & $O(1)$ & Classification risk \eqref{Rpn}\\ 
$\widehat{R}_{\rm{U^m}}(f)$ \cite{scott2020learning} & $\checkmark$ & $\checkmark$ & $\times$ & $O(m^3)$ & Balanced risk \eqref{balrisk}\\ 
$\widehat{R}_{\rm{prop}\text{-}\rm{c}}(f)$ \cite{tsai2020learning} & $\checkmark$ & $\times$ & $\checkmark$ & $O(1)$ & Proportion risk \eqref{proprisk}\\ 
\hline
\textbf{Proposed} & $\checkmark$ & $\checkmark$ & $\checkmark$ & $O(1)$ & Classification risk \eqref{Rpn}\\ 
\hline
\end{tabular}
\end{center}
\vspace{-0.5em}
\end{table*}
\vspace{-0.5em}%
\subsection{Learning from Fully Labeled Data}
Let $\mathcal{X}$ be the input feature space and $\mathcal{Y}=\{+1, -1\}$ be a binary label space,  $\boldsymbol{x} \in \mathcal{X}$ and $y \in \mathcal{Y}$ be the input and output random variables following an underlying joint distribution $\mathcal{D}$.
Let $f : \mathcal{X} \rightarrow \mathbb{R}$ be an arbitrary binary classifier, and $\ell_\mathrm{b}(t, y):\mathbb{R} \times \mathcal{Y} \rightarrow \mathbb{R}_+$ be a \emph{loss function} such that the value $\ell_\mathrm{b}(t, y)$ means the loss by predicting $t$ when the ground-truth is $y$. 
The goal of binary classification is to train a classifier $f$ that minimizes the \emph{risk} defined as
\begin{align}
R(f)&=\mathbb{E}_{(\boldsymbol{x},y)\sim\mathcal{D}}[\ell_\mathrm{b}(f(\boldsymbol{x}),y)]
\label{Rpn}
\end{align}
where $\mathbb{E}$ denotes the expectation.
For evaluation, $\ell_\mathrm{b}$ is often chosen as $\ell_\mathrm{01}(t,y)=(1-\sign((t-\frac{1}{2})\cdot y))/2$ and then the risk $R$ becomes the standard performance measure for classification, a.k.a. the \emph{classification error}.
For training, $\ell_\mathrm{01}$ is replaced by a \emph{surrogate loss},%
\footnote{The surrogate loss $\ell_\mathrm{s}$ should be \emph{classification-calibrated} so that the predictions can be the same for classifiers learned by using $\ell_\mathrm{s}$ and $\ell_\mathrm{01}$ \cite{bartlett06jasa}.} e.g., the logistic loss $\ell_\mathrm{log}(t,y)=\ln(1+\exp(-t\cdot y))$, since $\ell_\mathrm{01}$ is discontinuous and therefore difficult to optimize \citep{bendavid06jcss}.


In most cases, $R$ cannot be calculated directly because the joint distribution $\mathcal{D}$ is unknown to the learner.
Given the labeled training set $\mathcal{X}=\{(\boldsymbol{x}_i, y_i)\}_{i = 1}^n\stackrel{\mathrm{i.i.d.}}{\sim}\mathcal{D}$ with $n$ samples, \emph{empirical risk minimization}~(ERM) \cite{vapnik98SLT} is a common practice that computes an approximation of $R$ by
        \begin{equation}
            \widehat{R}(f)=\frac{1}{n}\sum\nolimits_{i=1}^n \ell_\mathrm{b}(f(\boldsymbol{x}_i),y_i). \label{empirical_risk}
        \end{equation}

\subsection{Learning from Multiple Sets of U Data}
Next, we consider U$^m$ classification.
We are given $m(m\geq 2)$ sets of unlabeled samples drawn from $m$ marginal densities $\{p_{\rm{tr}}^{j}(\boldsymbol{x})\}_{j=1}^{m}$, where
\begin{align}
p_{\rm{tr}}^{j}(\boldsymbol{x})=\pi_j p_\mathrm{p}(\boldsymbol{x}) + (1-\pi_j)p_\mathrm{n}(\boldsymbol{x}),
\label{ptr}
\end{align}
each $p_{\rm{tr}}^{j}(\boldsymbol{x})$ is seen as a mixture of the positive and negative class-conditional densities $(p_\mathrm{p}(\boldsymbol{x}), p_\mathrm{n}(\boldsymbol{x}))=(p(\boldsymbol{x}| y=+1), p(\boldsymbol{x}| y=-1))$, and $\pi_j=p_{\rm{tr}}^{j}(y = +1)$ denotes the class prior of the $j$-th U set.
Note that given only U data, it is \emph{theoretically impossible} to learn the class priors without any assumptions \cite{menon15icml}, so we assume all necessary class priors are given, which are the only \emph{weak supervision} we will leverage.%
\footnote{By introducing the \emph{mutually irreducible assumption} \cite{scott2013classification}, the class priors become identifiable and can be estimated in some cases, see \citet{menon15icml}, \citet{liu16tpami}, \citet{jain2016estimating}, and \citet{yao2020dual} for details.}
To make the problem mathematically solvable, among the $m$ sets of U data, we also assume that at least two of them are different, i.e., $\exists j, j' \in \{1,\ldots,m\}$ such that $j \neq j' $ and $\pi_j \neq \pi_{j'}$.

In contrast to supervised classification where we have a fully labeled training set $\mathcal{X}$ directly drawn from $\mathcal{D}$, now we only have access to $m$ sets of U data $\mathcal{X}_{\rm{tr}}=\{\mathcal{X}_{\rm{tr}}^{j}\}_{j=1}^{m}$, where
\begin{equation}
      \mathcal{X}_{\rm{tr}}^{j} = \{\boldsymbol{x}_{1}^j,\ldots, \boldsymbol{x}_{n_j}^j\} \overset{\rm{i.i.d.}}{\sim} p_{\rm{tr}}^{j}(\boldsymbol{x}),
\label{kusets}
\end{equation}
and $n_j$ denotes the sample size of the $j$-th U set.
But our goal is still the same as supervised classification: to obtain a binary classifier that generalizes well with respect to $\mathcal{D}$, despite the fact that it is unobserved.

\subsection{Related Work}
Here, we review some related works for U$^m$ classification.

\paragraph{Clustering methods}
Learning from only U data is previously regarded as
\emph{discriminative clustering} \citep{xu04nips,gomes10nips}.
However, these methods are often suboptimal since they rely on a critical assumption that \emph{one cluster exactly corresponds to one class}, 
and hence even perfect clustering may still result in poor classification.
As a consequence, we prefer ERM to clustering.

\paragraph{Proportion risk methods}
The U$^m$ classification setting is also related to \emph{learning with label proportions}~(LLP), with a subtle difference in the experimental design.%
\footnote{The majority of LLP papers use uniform sampling for bag generation, which may result in the same label proportion for all the U sets and make the LLP problem computationally intractable \cite{scott2020learning}.
Our simulation in Sec.~\ref{sec:exp} avoids the issue.}
However, most LLP methods are not ERM-based, but based on the following \emph{empirical proportion risk}~(EPR) \cite{yu2014learning}:
\begin{align}
\widehat{R}_{\rm{prop}}(f)=\sum\nolimits_{j=1}^{m}d_{\rm{prop}}(\pi_j, \hat{\pi}_j),
\label{proprisk}%
\end{align}
where $\pi_j$ and
$\hat{\pi}_j=\frac{1}{n_j}\sum_{i=1}^{n_j}(1+\sign(f(\boldsymbol{x}_i^j)-1/2))/2$
are the true and predicted label proportions for the $j$-th U set $\mathcal{X}_{\rm{tr}}^{j}$, and $d_{\rm{prop}}$ is a distance function.
State-of-the-art method in this line combined EPR with consistency regularization and proposed the following learning objective:
\begin{align}
\widehat{R}_{\rm{prop}\text{-}\rm{c}}(f)=\widehat{R}_{\rm{prop}}(f)+\alpha\ell_{\rm{cons}}(f),
\label{llpvat}
\end{align}
where $\ell_{\rm{cons}}(f)=d_{\rm{cons}}(f(\boldsymbol{x}), f(\boldsymbol{\hat{x}}))$ is the consistency loss given a distance function $d_{\rm{cons}}$ and $\boldsymbol{\hat{x}}$ is a perturbed input from the original one $\boldsymbol{x}$ \cite{tsai2020learning}.

\paragraph{Classification risk methods}
A breakthrough of the ERM-based method for U$^m$ classification is \citet{lu2018minimal} which assumed $m=2$ and $\pi_1>\pi_2$, and proposed an equivalent expression of the classification risk \eqref{Rpn}:
\begin{align*}
&R_{\rm{U^2}}(f)=
\underbrace{
\mathbb{E}_{\boldsymbol{x}\sim p_{\rm{tr}}^{1}}
c_1^+\ell_\mathrm{b}(f(\boldsymbol{x}),+1)
-\mathbb{E}_{\boldsymbol{x}\sim p_{\rm{tr}}^{2}}
c_2^+\ell_\mathrm{b}(f(\boldsymbol{x}),+1)}_{R_{\rm{U^2}\text{-}\rm{p}}(f)}\notag\\
&~~~\underbrace{
-\mathbb{E}_{\boldsymbol{x}\sim p_{\rm{tr}}^{1}}
c_1^-\ell_\mathrm{b}(f(\boldsymbol{x}),-1)
+\mathbb{E}_{\boldsymbol{x}\sim p_{\rm{tr}}^{2}}
c_2^-\ell_\mathrm{b}(f(\boldsymbol{x}),-1)}_{R_{\rm{U^2}\text{-}\rm{n}}(f)},
\end{align*}
where $c_1^+=\frac{(1-\pi_2)\pi_{\mathcal{D}}}{\pi_1-\pi_2}$, $c_1^-=\frac{\pi_2(1-\pi_{\mathcal{D}})}{\pi_1-\pi_2}$, $c_2^+=\frac{(1-\pi_1)\pi_{\mathcal{D}}}{\pi_1-\pi_2}$, $c_2^-=\frac{\pi_1(1-\pi_{\mathcal{D}})}{\pi_1-\pi_2}$, and $\pi_{\mathcal{D}}$ denotes the class prior of the test set.
If $\pi_{\mathcal{D}}$ is assumed to be $\frac{1}{2}$ in $R_{\rm{U^2}}(f)$, the obtained $R_{\rm{U^2}\text{-}\rm{b}}(f)$ \cite{menon15icml} corresponds to the balanced risk, a.k.a. the \emph{balanced error} \cite{brodersen10icpr}:
\begin{align}
R_{\rm{b}}(f)&=\frac{1}{2}\mathbb{E}_{\boldsymbol{x}\sim p_\mathrm{p}}[\ell_\mathrm{b}(f(\boldsymbol{x}),+1)]+\frac{1}{2}\mathbb{E}_{\boldsymbol{x}\sim p_\mathrm{n}}[\ell_\mathrm{b}(f(\boldsymbol{x}),-1)],
\label{balrisk}
\end{align}
where $\ell_\mathrm{b}$ is $\ell_\mathrm{01}$.
Note that $R_{\rm{b}}(f) =R(f)$ for any $f$ if and only if $\pi_{\mathcal{D}}=\frac{1}{2}$, which means that it definitely biases learning when $\pi_{\mathcal{D}}\approx\frac{1}{2}$ is not the case.
Given $\mathcal{X}_{\rm{tr}}^{1}$ and $\mathcal{X}_{\rm{tr}}^{2}$, $R_{\rm{U^2}}(f)$ and $R_{\rm{U^2}\text{-}\rm{b}}(f)$ can be approximated by their empirical counterparts $\widehat{R}_{\rm{U^2}}(f)$ and $\widehat{R}_{\rm{U^2}\text{-}\rm{b}}(f)$.

It is shown in \citet{lu2020mitigating} that the empirical training risk $\widehat{R}_{\rm{U^2}}(f)$ can take negative values which causes overfitting, so they proposed a corrected learning objective that wraps the empirical risks of the positive class $\widehat{R}_{\rm{U^2}\text{-}\rm{p}}(f)$ and the negative class $\widehat{R}_{\rm{U^2}\text{-}\rm{n}}(f)$ into some non-negative correction function $f_\mathrm{c}$, such that $f_\mathrm{c}(x)=x$ for all $x\ge0$ and $f_\mathrm{c}(x)>0$ for all $x<0$:
$\widehat{R}_{\rm{U^2}\text{-}\rm{c}}(f)= f_\mathrm{c}(\widehat{R}_{\rm{U^2}\text{-}\rm{p}}(f))+f_\mathrm{c}(\widehat{R}_{\rm{U^2}\text{-}\rm{n}}(f))$.
Note that $\widehat{R}_{\rm{U^2}\text{-}\rm{c}}$ is biased with finite samples, but \citet{lu2020mitigating} showed its risk-consistency, i.e., it converges to the original risk $R$ in \eqref{Rpn} if $n_1,n_2\to\infty$.

Although these risk-consistent methods are advantageous in terms of flexibility and theoretical guarantees, they are limited to $2$ U sets.
Recently, \citet{scott2020learning} extended the previous method for the general $m (m\geq2)$ setting.
More specifically, they assumed the number of sets $m=2k$ and proposed a pre-processing step that finds $k$ pairs of the U sets by solving a maximum weighted matching problem \cite{edmonds1965maximum}.
Then they linearly combine the unbiased balanced risk estimator of each pair.%
\footnote{In \citet{scott2020learning}, it is assumed that $\pi_1\neq\pi_2$ in each pair, which is a stronger assumption than ours.}
The resulted weighted learning objective is given by
\begin{equation}
\widehat{R}_{\rm{U^m}}(f)=\sum\nolimits_{j=1}^{k}\omega_j\widehat{R}_{\rm{U^2}\text{-}\rm{b}}(f).
\label{muub}
\end{equation}
This method is promising but has some practical issues: the pairing step is computationally very inefficient and the weights are hard to tune in practice.

A comparison of the previous works with our proposed method that will be introduced in Sec.~\ref{sec:U$^m$ classification} is given in Table~\ref{methods}.

\section{U$^m$ classification via Surrogate Set Classification}
\label{sec:U$^m$ classification}
In this section, we propose a new ERM-based method for learning from multiple U sets via a surrogate set classification task and analyze it theoretically.
All the proofs are given in Appendix~\ref{sec:proof}.

\subsection{Surrogate Set Classification Task}
The main challenge in the U$^m$ classification problem is that we have \emph{no} access to the ground-truth labels of the training examples so that the empirical risk \eqref{empirical_risk} in supervised binary classification cannot be computed directly.
Our idea is to consider a surrogate set classification task that could be tackled easily from the given U sets.
It serves as a proxy and gives us a classifier-consistent solution to the original binary classification problem.
    


Specifically, denote by $\Bar{y}\in \{1,2,\ldots, m\}$ the index of the U set, i.e., the index of the corresponding marginal density. 
By treating $\bar{y}$ as a \emph{surrogate label}, we formulate the surrogate set classification task as the standard multi-class classification.
Let $\Bar{\mathcal{D}}$ be the joint distribution for the random variables $\boldsymbol{x} \in \mathcal{X}$ and $\bar{y} \in \bar{\mathcal{Y}} = \{1, 2,\ldots,m\}$.
Any $\bar{\mathcal{D}}$ can be identified via the class priors $\{\rho_j=p(\Bar{y}=j)\}_{j=1}^{m}$ and the class-conditional densities $\{p(\boldsymbol{x} \mid\Bar{y}=j)= p_{\rm{tr}}^j(\boldsymbol{x})\}_{j=1}^{m}$, where $\rho_j$ can be estimated by $\rho_j=\frac{n_j}{\sum_{j=1}^{m}{n_j}}$.

The goal of surrogate set classification is to train a classifier $\boldsymbol{g}(\boldsymbol{x}):\mathcal{X}\rightarrow \mathbb{R}^m$ that minimizes the following risk:
    \begin{align}
    \label{surrorisk}
        {R_{\rm{surr}}(\boldsymbol{g})=\mathbb{E}_{(\boldsymbol{x},\bar{y})\sim\bar{\mathcal{D}}}[\ell(\boldsymbol{g}(\boldsymbol{x}),\bar{y})]},
    \end{align}
where $\ell(\boldsymbol{g}(\boldsymbol{x}), \bar{y}): \mathbb{R}^m \times \bar{\mathcal{Y}} \rightarrow \mathbb{R}_+$ is a proper loss for $m$-class classification, e.g., the \emph{cross-entropy loss}:
    \begin{equation} \nonumber
        \ell_{\rm{ce}}({\boldsymbol{g}}(\boldsymbol{x}), \Bar{y})=-\sum_{j=1}^m\boldsymbol{1}(\Bar{y}=j)\log(g_j(\boldsymbol{x}))=-\log(g_{\bar{y}}(\boldsymbol{x})),
    \end{equation}
where $\boldsymbol{1}(\cdot)$ is the indicator function, $g_j(\boldsymbol{x})$ is the $j$-th element of $\boldsymbol{g}(\boldsymbol{x})$, and is a score function that estimates the true class-posterior probability $\Bar{\eta}_j(\boldsymbol{x}) = p(\Bar{y} = j \mid \boldsymbol{x})$.
Typically, the predicted label $\bar{y}_{\rm{pred}}$ takes the form $\bar{y}_{\rm{pred}} = \mathop{\argmax}_{j\in [m]}g_j(\boldsymbol{x})$.

Now the unlabeled training sets given by \eqref{kusets} for the binary classification can be seen as a labeled training set ${\mathcal{X}}_{\rm{tr}} = \{(\boldsymbol{x}_i, \Bar{y}_i)\}_{i=1}^{n_{\rm{tr}}}\stackrel{\mathrm{i.i.d.}}{\sim}\Bar{\mathcal{D}}$ for the $m$-class classification, where $n_{\rm{tr}}=\sum_{j=1}^{m}{n_j}$ is the total number of U data.
We can use $\mathcal{X}_{\rm{tr}}$ to approximate the risk $R_{\rm{surr}}$ by
    \begin{equation}
        \widehat{R}_{\rm{surr}}(\boldsymbol{g})
        =\frac{1}{n_{\rm{tr}}}\sum\nolimits_{i=1}^{n_{\rm{tr}}}{\ell}({\boldsymbol{g}}(\boldsymbol{x}_i),\bar{y}_i).
        \label{surrog}
    \end{equation}

\subsection{Bridge Two Posterior Probabilities}
Let $\eta(\boldsymbol{x}) = p(y=+1 \mid \boldsymbol{x})$ be the class-posterior probability for class $+1$ in the original binary classification problem, and $\Bar{\eta}_j(\boldsymbol{x}) = p(\Bar{y} = j \mid \boldsymbol{x})$ be the class-posterior probability for class $j$ in the surrogate set classification problem. 
We theoretically bridge them by the following theorem.
\begin{theorem} 
\label{thm:transformation}%
By the definitions of $\mathcal{D}$, $\eta(\boldsymbol{x})$, $\bar{\mathcal{D}}$, and $\bar{\eta}_j(\boldsymbol{x})$, we have
\begin{equation}
\Bar{\eta}_j(\boldsymbol{x})=T_j(\eta(\boldsymbol{x})),\quad \forall j = 1,\ldots,m,
\end{equation}
where
\begin{equation} \nonumber
T_j(\eta(\boldsymbol{x}))=\displaystyle\frac{a_j\cdot\eta(\boldsymbol{x})+b_j}{c\cdot\eta(\boldsymbol{x})+d},
\end{equation}
$a_j = {\rho}_j(\pi_j-\pi_{\mathcal{D}})$, $b_j = {\rho}_j\pi_{\mathcal{D}}(1-\pi_j)$, $c = \sum_{j=1}^m{\rho}_j(\pi_j-\pi_{\mathcal{D}})$, and $d = \sum_{j=1}^m{\rho}_j\pi_{\mathcal{D}}(1-\pi_j)$.
\label{theorem1}
\end{theorem}

Such a relationship has been previously studied by \citet{menon15icml} in the context of \emph{corrupted label learning} for a specific $2 \times 2$ case, i.e., 2 clean classes are transformed to 2 corrupted classes, and they used $T_j(\cdot)$ to post-process the threshold of the score function learned from corrupted data.
Our proposal can be regarded as its extension to a general $2 \times m$ case and $T_j(\cdot)$ is used to connect the original binary classifier with the surrogate multi-set-class classifier.

Let $\boldsymbol{T}(\cdot): \mathbb{R} \rightarrow \mathbb{R}^m$ be a vector form of the transition function $\boldsymbol{T}(\cdot) = [T_1(\cdot),\ldots,T_m(\cdot)]^\top$.
Note that the coefficients in $T_j(\cdot)$ are all constants and $\boldsymbol{T}(\cdot)$ is deterministic.
Next, we study properties of the transition function $\boldsymbol{T}$ in the following lemma, which implies the feasibility of approaching $\eta(\boldsymbol{x})$ by means of estimating $\Bar{\eta}_j(\boldsymbol{x})$. 
\begin{lemma}
The transition function $\boldsymbol{T}(\cdot)$ is an injective function in the domain $[0, 1]$.
\label{lemma_injective}
\end{lemma}

\subsection{Classifier-consistent Algorithm}
Given the transition function $\boldsymbol{T}$, we have two choices to obtain $\eta(\boldsymbol{x})$ from $\Bar{\eta}_j(\boldsymbol{x})$. 
First, one can estimate $\Bar{\eta}_j(\boldsymbol{x})$, then calculate $\eta(\boldsymbol{x})$ via the inverse function $T_j^{-1}(\Bar{\eta}_j(\boldsymbol{x}))$. 
Second, one can encode $\eta(\boldsymbol{x})$ as a latent variable into the computation of $\Bar{\eta}_j(\boldsymbol{x})$ and obtain both of them simultaneously. 
We prefer the latter for three reasons.
\begin{itemize}
\item Computational efficiency: the latter is a one-step solution and avoids additional computations of the inverse functions, which provides computational efficiency and easiness for implementation.
\item Robustness: since the coefficients of $T_j(\cdot)$ may be perturbed by some noise in practice, its inversion in the former method may enlarge the noise by orders of magnitude, making the learning process less robust.
\item Identifiability:
calculating $T_j^{-1}(\Bar{\eta}_j(\boldsymbol{x}))$ in the former method for all $j=\{1,\ldots,m\}$ induces $m$ estimates of $\eta(\boldsymbol{x})$, and they are usually non-identical due to the estimation error of $\Bar{\eta}_j(\boldsymbol{x})$ from finite samples or noisy $T_j(\cdot)$, causing a new non-identifiable problem.
\end{itemize}

    \begin{algorithm}[t]
    \caption{U$^m$-SSC based on stochastic optimization}
     \label{alg}
    \hspace*{0.02in} {\bf Input:}
    Model $f$, $m$ sets of unlabeled data $\mathcal{X}_{\rm{tr}}$, class priors $\{\pi_j\}_{j=1}^m$ and $\pi_{\mathcal{D}}$
    \begin{algorithmic}[1]
    \STATE Compute $a_j$, $b_j$, $c$ and $d$ of $\boldsymbol{T}(\cdot) = [T_1(\cdot),\ldots,T_m(\cdot)]^\top$ in Theorem~\ref{theorem1} using $\{\pi_j\}_{j=1}^m$ and $\pi_{\mathcal{D}}$.
    \STATE Let $\boldsymbol{g}=\boldsymbol{T}(f)$ and $\mathcal{A}$ be a SGD-like optimizer working on $\boldsymbol{g}$.
    \FOR{$t = 1, 2,\ldots, $number\_of\_epochs}
        \STATE Shuffle $\mathcal{X}_{\rm{tr}}$
        \FOR{$i = 1, 2,\ldots, $number\_of\_mini-batches}
            \STATE Fetch mini-batch $\Bar{\mathcal{X}}_{\rm{tr}}$ from $\mathcal{X}_{\rm{tr}}$
            \STATE Forward $\Bar{\mathcal{X}}_{\rm{tr}}$ and get $f(\Bar{\mathcal{X}}_{\rm{tr}})$
            \STATE Compute $\boldsymbol{g}\left(\Bar{\mathcal{X}}_{\rm{tr}}\right) = \boldsymbol{T}\left(f(\Bar{\mathcal{X}}_{\rm{tr}})\right)$
            \STATE Compute loss by \eqref{surrog} using $\boldsymbol{g}(\Bar{\mathcal{X}}_{\rm{tr}})$
            \STATE Update $\boldsymbol{g}$ by $\mathcal{A}$, which induces an update on $f$
        \ENDFOR
    \ENDFOR
    \end{algorithmic}
    \hspace*{0.02in} {\bf Output: $f$}
    \end{algorithm}

Therefore, we choose to embed the estimation of $\eta(\boldsymbol{x})$ into the estimation of $\Bar{\eta}_j(\boldsymbol{x})$.
More specifically, let $f(\boldsymbol{x})$ be the model output that estimates $\eta(\boldsymbol{x})$, then we make use of the transition function $T_j(\cdot)$ and model $g_j(\boldsymbol{x})=T_j\left(f(\boldsymbol{x})\right)$. 
Based on it, we propose to learn with the following modified loss function:
    \begin{equation}
    {\ell}(\boldsymbol{g}(\boldsymbol{x}), \bar{y})=\ell(\boldsymbol{{T}}(f(\boldsymbol{x})), \bar{y}), \label{l_surr}
    \end{equation}
where $\boldsymbol{T}(f(\boldsymbol{x}))=[T_1(f(\boldsymbol{x})),\ldots,T_m(f(\boldsymbol{x}))]^\top$. 
Then the corresponding risk for the surrogate task can be written as
\begin{align}
R_{\rm{surr}}(f)&=\mathbb{E}_{(\boldsymbol{x},\Bar{y})\sim\Bar{\mathcal{D}}}[\ell(\boldsymbol{T}(f(\boldsymbol{x})), \bar{y})]\notag\\
&=\mathbb{E}_{(\boldsymbol{x},\Bar{y})\sim\Bar{\mathcal{D}}}[{\ell}(\boldsymbol{g}(\boldsymbol{x}), \bar{y})]= R_{\rm{surr}}(\boldsymbol{g}),
\label{R_surr}
\end{align}
and an equivalent expression of the empirical risk \eqref{surrog} is given by
\begin{equation}
        \widehat{R}_{\rm{surr}}(f)
        =\frac{1}{n_{\rm{tr}}}\sum\nolimits_{i=1}^{n_{\rm{tr}}}{\ell}(\boldsymbol{T}(f(\boldsymbol{x}_i)),\bar{y}_i).
        \label{surrof}
\end{equation}
    
In order to prove that this method is classifier-consistent, we introduce the following lemma.
\begin{lemma}
Let $\bar{\boldsymbol{\eta}}(\boldsymbol{x}) = [\Bar{\eta}_1(\boldsymbol{x}),\ldots,\Bar{\eta}_m(\boldsymbol{x})]^\top$ and $\boldsymbol{g}^\star(\boldsymbol{x})=\argmin_{\boldsymbol{g}}R_{\rm{surr}}(\boldsymbol{g};\ell)$ be the optimal classifier of \eqref{surrorisk}.
Provided that a proper loss function, e.g., the cross-entropy loss or mean squared error, is chosen for $\ell$, we have $\boldsymbol{g}^{\star}(\boldsymbol{x})=\bar{\boldsymbol{\eta}}(\boldsymbol{x})$.
\label{lemma_g_star}
\end{lemma}
Since $\boldsymbol{g}(\boldsymbol{x})=\boldsymbol{T}(f(\boldsymbol{x}))$ and $\boldsymbol{T}(\cdot)$ is deterministic, when considering minimizing $R_{\rm{surr}}(f)$ that takes $f$ as the argument, we can prove the following classifier-consistency.
\begin{theorem} [Identification of the optimal binary classifier]
Assume that the cross-entropy loss or mean squared error is used for $\ell$ and $\ell_\mathrm{b}$,
and the model $\mathcal{G}$ used for learning $\boldsymbol{g}$ is very flexible, e.g., deep neural networks,
so that $\boldsymbol{g}^{\star}\in\mathcal{G}$. 
Let $f_{\rm{surr}}^{\star}$ be the U$^m$-SSC optimal classifier induced by $\boldsymbol{g}^{\star}$, 
and $f^{\star}=\argmin_{f}R(f;\ell_\mathrm{b})$ be the optimal classifier of \eqref{Rpn}, 
we have $f^{\star}_{\rm{surr}}=f^{\star}$.
\label{theorem2}
\end{theorem}  

\begin{figure}[t]
\centerline{\includegraphics[scale=0.48]{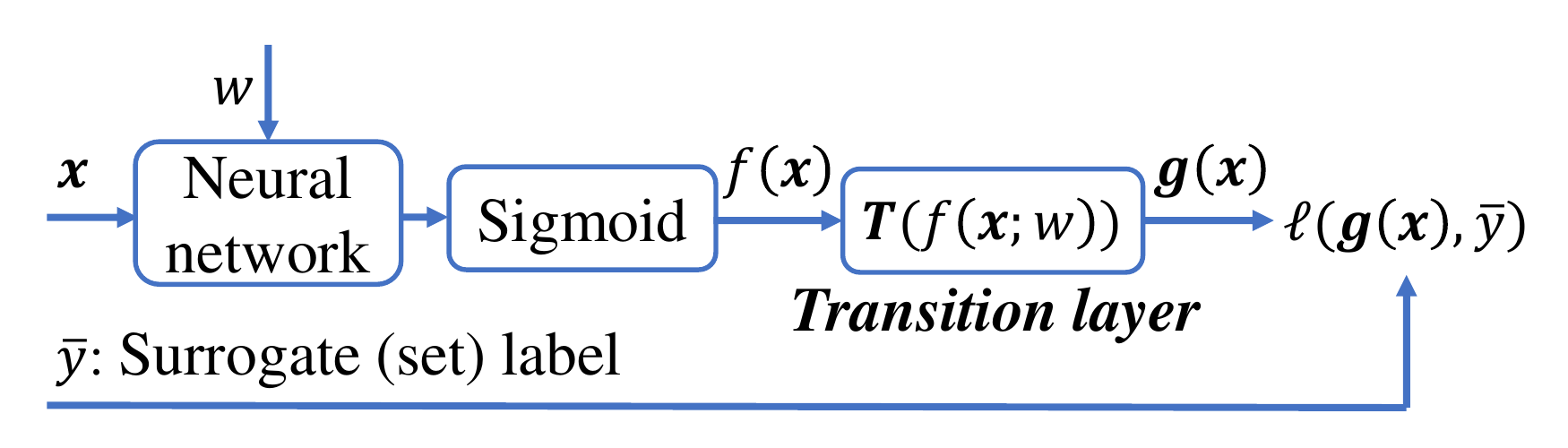}}
\caption{Implementation diagram of U$^m$-SSC.}
\label{implementation}
\end{figure}
    
So far, we have proved that the optimal classifier for the original binary classification task can be identified by the U$^m$-SSC learning scheme.
Its algorithm is described in Algorithm~\ref{alg} and its implementation is illustrated in Figure~\ref{implementation}.

We implement $\boldsymbol{T}(\cdot)$ by adding a transition layer following the sigmoid function of the neural network~(NN).
At the training phase, a sample $(\boldsymbol{x}_{\rm{tr}},\bar{y}_{\rm{tr}})$ is fetched to the network. 
A sigmoid function $f_{\rm{sig}}(x)=\frac{1}{1+e^{-x}}$ is used to map the output of NN to the range $[0, 1]$ such that the output $f(\boldsymbol{x})$ is an estimate of $\eta(\boldsymbol{x})$. 
Then $f(\boldsymbol{x})$ is forwarded to the transition layer and a vector output $\boldsymbol{g}(\boldsymbol{x})=\boldsymbol{T}(f(\boldsymbol{x}))$ is obtained.
The loss computed on the output $\boldsymbol{g}(\boldsymbol{x})$ and the surrogate label $\bar{y}_{\rm{tr}}$ by \eqref{surrog} is then used for updating the NN weights $w$. 
Note that the transition layer is fixed and only the weights in the base network are learnable.
At the test phase, for any test sample $\boldsymbol{x}_{\rm{te}}$, we compute $f(\boldsymbol{x}_{\rm{te}})$ using only the trained base network and sigmoid function.
The test sample is classified by using the sign function, i.e., ${\rm{sign}}(f(\boldsymbol{x}_{\rm{te}}) - \frac{1}{2})$.
Our proposed method is model-agnostic and can be easily trained with a stochastic optimization algorithm, which ensures its scalability to large-scale datasets.

    \begin{table*}[t]\centering
    \caption{Specification of datasets and corresponding models.}
        \vspace{1ex}%
        \newcommand{\tabincell}[2]{\begin{tabular}{@{}#1@{}}#2\end{tabular}}
        \begin{tabular}{ l|r r r r|l } 
        \hline
        Dataset & \tabincell{c}{\# Train} & \tabincell{c}{\# Test} & \tabincell{c}{\# Features} & \tabincell{c}{$\pi_{\mathcal{D}}$} & \tabincell{c}{Model} \\ 
        \hline
        MNIST \cite{lecun98mnist} & 60,000 & 10,000 & 784 & 0.49 & 5-layer MLP \\ 

        Fashion-MNIST \cite{xiao17arxiv} & 60,000 & 10,000 & 784 & 0.8 & 5-layer MLP \\ 
        Kuzushiji-MNIST \cite{clanuwat2018deep} & 60,000 & 10,000 & 784 & 0.3 & 5-layer MLP \\ 
        CIFAR-10 \cite{krizhevsky09cifar} & 50,000 & 10,000 & 3,072 & 0.7 & ResNet-32 \\ 
        \hline
        \end{tabular}
        \label{dataset}
    \end{table*} 

\subsection{Theoretical Analysis}
In what follows, we upper-bound the estimation error of our proposed method.
Let $\hat{f}_{\rm{surr}}=\argmin_{f\in\mathcal{F}}\widehat{R}_{\rm{surr}}(f)$ be our empirical classifier, where $\mathcal{F}=\{f:\mathcal{X} \rightarrow\mathbb{R}\}$ is a class of measurable functions, and $f^{\star}_{\rm{surr}}=\argmin_{f\in\mathcal{F}}R_{\rm{surr}}(f)$ be the optimal classifier, the estimation error is defined as the gap between the risk of $\hat{f}_{\rm{surr}}$ and that of $f^{\star}_{\rm{surr}}$, i.e., $R_{\rm{surr}}(\hat{f}_{\rm{surr}})-R_{\rm{surr}}(f^{\star}_{\rm{surr}})$.
To derive the estimation error bound, we firstly investigate the Lipschitz continuity of the transition function $\boldsymbol{T}(f(\boldsymbol{x}))$.
\begin{lemma}
Assume that among the $m$ sets of U data, at least two of them are different, i.e., $\exists j, j' \in \{1,\ldots,m\}$ such that $j \neq j' $ and $\pi_j \neq \pi_{j'}$,
and $0\leq f(\boldsymbol{x})\leq1,\forall x \in \mathcal{X}$, e.g., $f(\boldsymbol{x})$ is mapped to $[0, 1]$ by the sigmoid function. 
Then, $\forall j = 1,\ldots,m$, the function $T_j(f(\boldsymbol{x}))$ is Lipschitz continuous w.r.t. $f(\boldsymbol{x})$ with a Lipschitz constant $2/\alpha^2$, where
\begin{equation} \nonumber
\alpha = \min\left(\sum_{j=1}^m{\rho}_j\pi_j(1-\pi_{\mathcal{D}}),\sum_{j=1}^m{\rho}_j\pi_{\mathcal{D}}(1-\pi_j)\right).
\end{equation}
\label{lemma4}
\end{lemma}
\vspace{-2ex}%
Then we analyze the estimation error as follows.
\begin{theorem} [Estimation error bound]
Assume that the loss $\ell(\boldsymbol{T}(f), \bar{y})$ is upper-bounded by $M_{\ell}$ and is $\mathcal{L_{\ell}}$-Lipschitz continuous w.r.t. $\boldsymbol{T}(f)$.
Let $\mathfrak{R}_{n_{\rm{tr}}}(\mathcal{F})$ be the \emph{Rademacher complexity} of $\mathcal{F}$ \cite{mohri12FML,sshwartz14UML}.
Then, for any $\delta>0$, we have with probability at least $1 - \delta$,
\begin{align}
&R_{\rm{surr}}(\hat{f}_{\rm{surr}}) - R_{\rm{surr}}(f^{\star}_{\rm{surr}}) \leq\notag\\ &\quad \frac{8\sqrt{2}m\mathcal{L}_{\ell}}{\alpha^2}\mathfrak{R}_{n_{\rm{tr}}}(\mathcal{F})+2M_{\ell}\sqrt{\frac{\ln(2/\delta)}{2n_{\rm{tr}}}}.
\end{align}
\label{error_bound}
\end{theorem}
\vspace{-2ex}%
Theorem \ref{error_bound} demonstrates that as the number of training samples goes to infinity, the risk of $\hat{f}_{\rm{surr}}$ converges to the risk of $f^{\star}_{\rm{surr}}$, since $\mathfrak{R}_{n_{\rm{tr}}}(\mathcal{F})\to0$ for all parametric models with a bounded norm.
Moreover, the coefficient $\alpha$ implies that a tighter error bound could be obtained when the class priors $\pi_j$ are close to 0 or 1. 
This conclusion agrees with our intuition that purer U sets (containing almost only positive/ negative examples) lead to better performance. 


\begin{figure*}[htbp]
    \begin{minipage}[c]{0.1\textwidth}~\end{minipage}\hspace{1.6em}%
    \begin{minipage}[c]{0.25\textwidth}\centering\normalsize Test error, 10 U sets \end{minipage}\hspace{2.65em}%
    \begin{minipage}[c]{0.25\textwidth}\centering\normalsize Test error, 50 U sets \end{minipage}\hspace{2.45em}%
    \begin{minipage}[c]{0.25\textwidth}\centering\normalsize Training risk, 50 U sets \end{minipage}\\
    \vspace{0.5em}
    \begin{minipage}[c]{0.08\textwidth}\flushright\small MNIST \end{minipage}%
    \hspace{0.5em}
    \begin{minipage}[c]{0.90\textwidth}
        \includegraphics[width=0.33\textwidth]{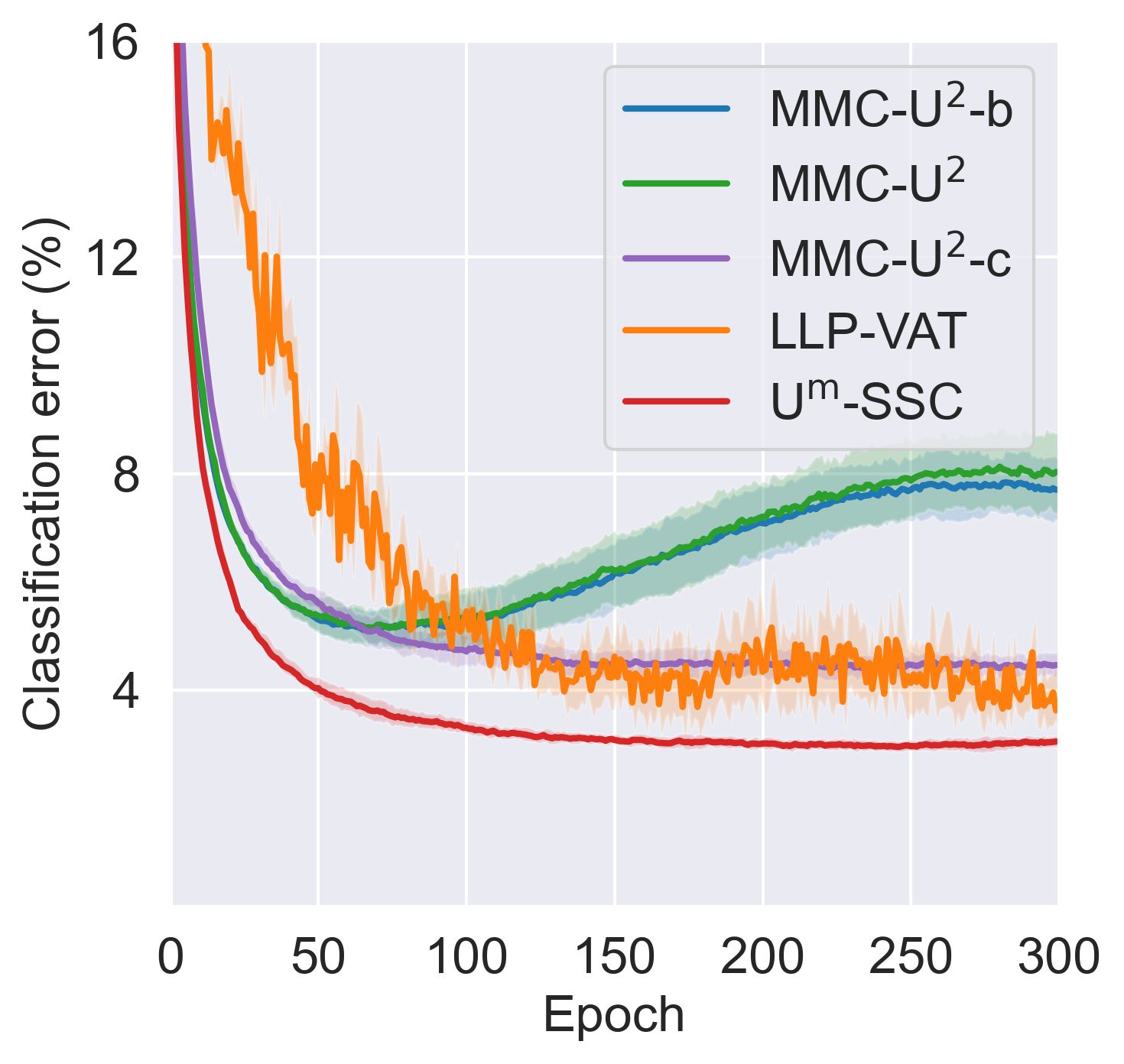}%
        \includegraphics[width=0.33\textwidth]{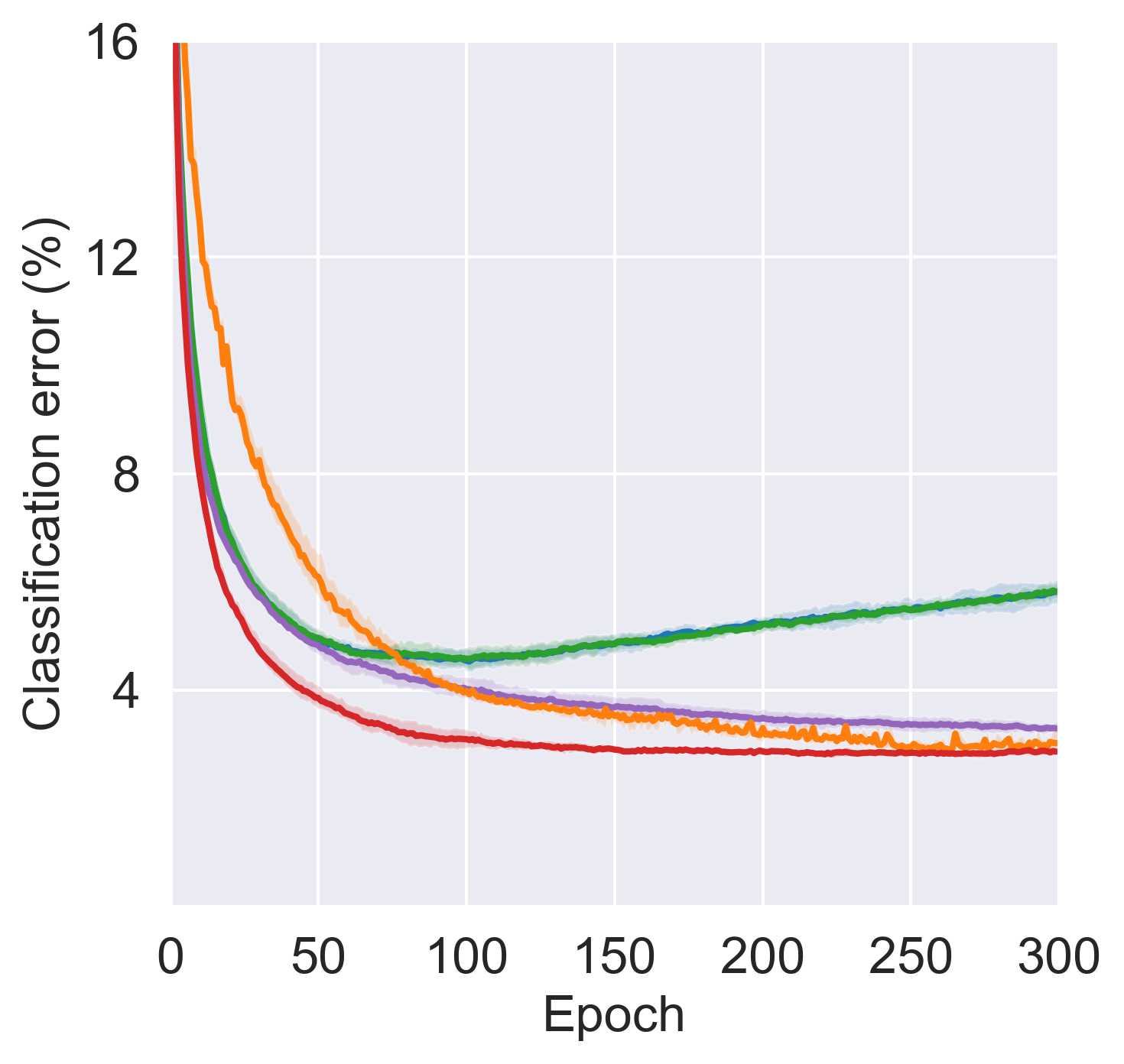}
        \includegraphics[width=0.33\textwidth]{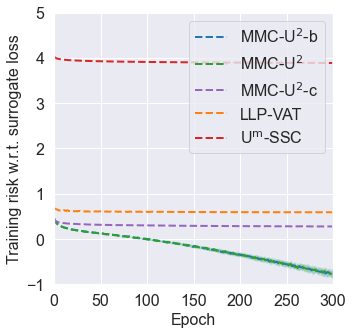}
    \end{minipage}\\
    \begin{minipage}[c]{0.08\textwidth}\flushright\small Fashion\\-MNIST \end{minipage}%
    \hspace{0.5em}
    \begin{minipage}[c]{0.90\textwidth}
        \includegraphics[width=0.33\textwidth]{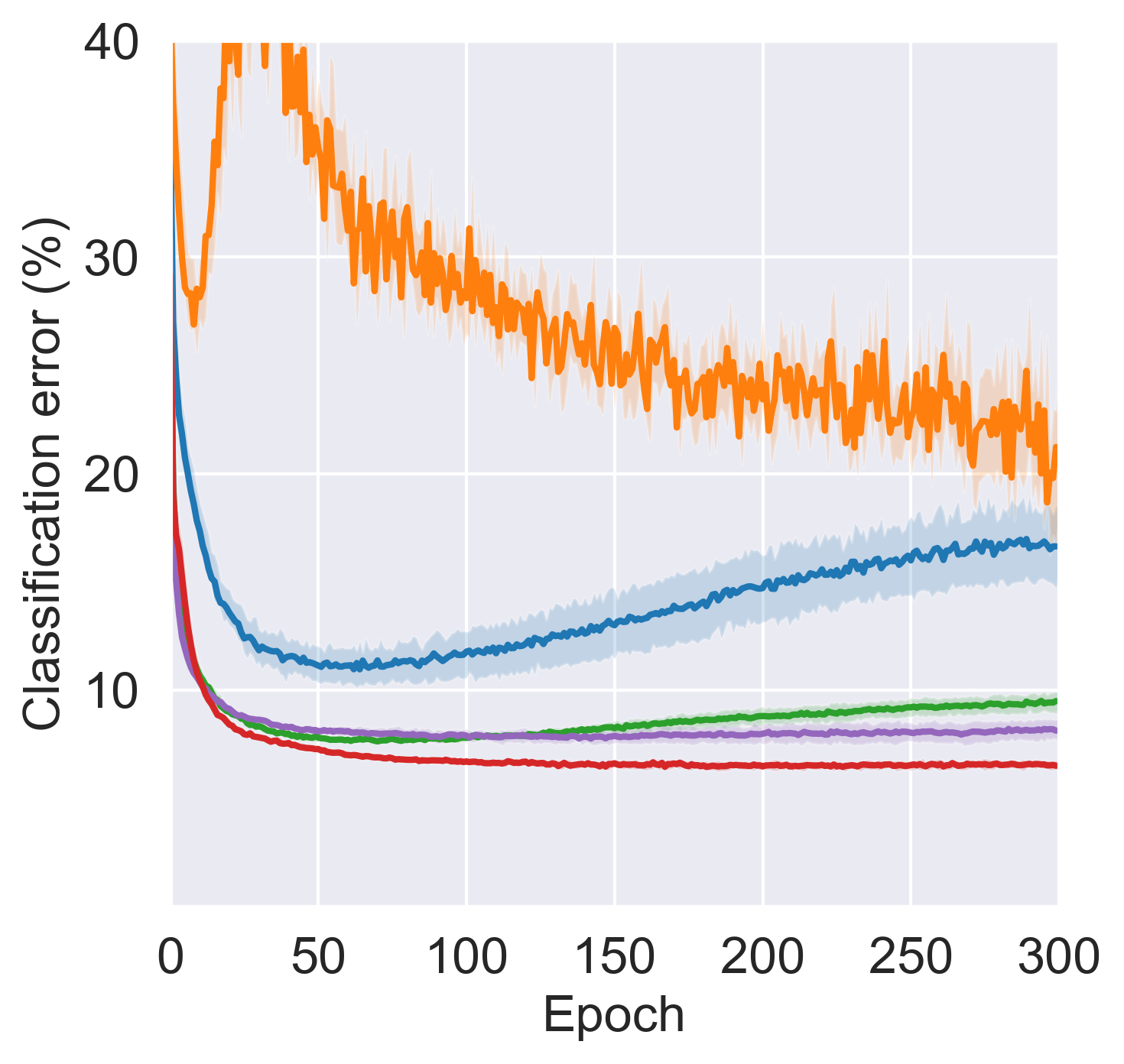}%
        \includegraphics[width=0.33\textwidth]{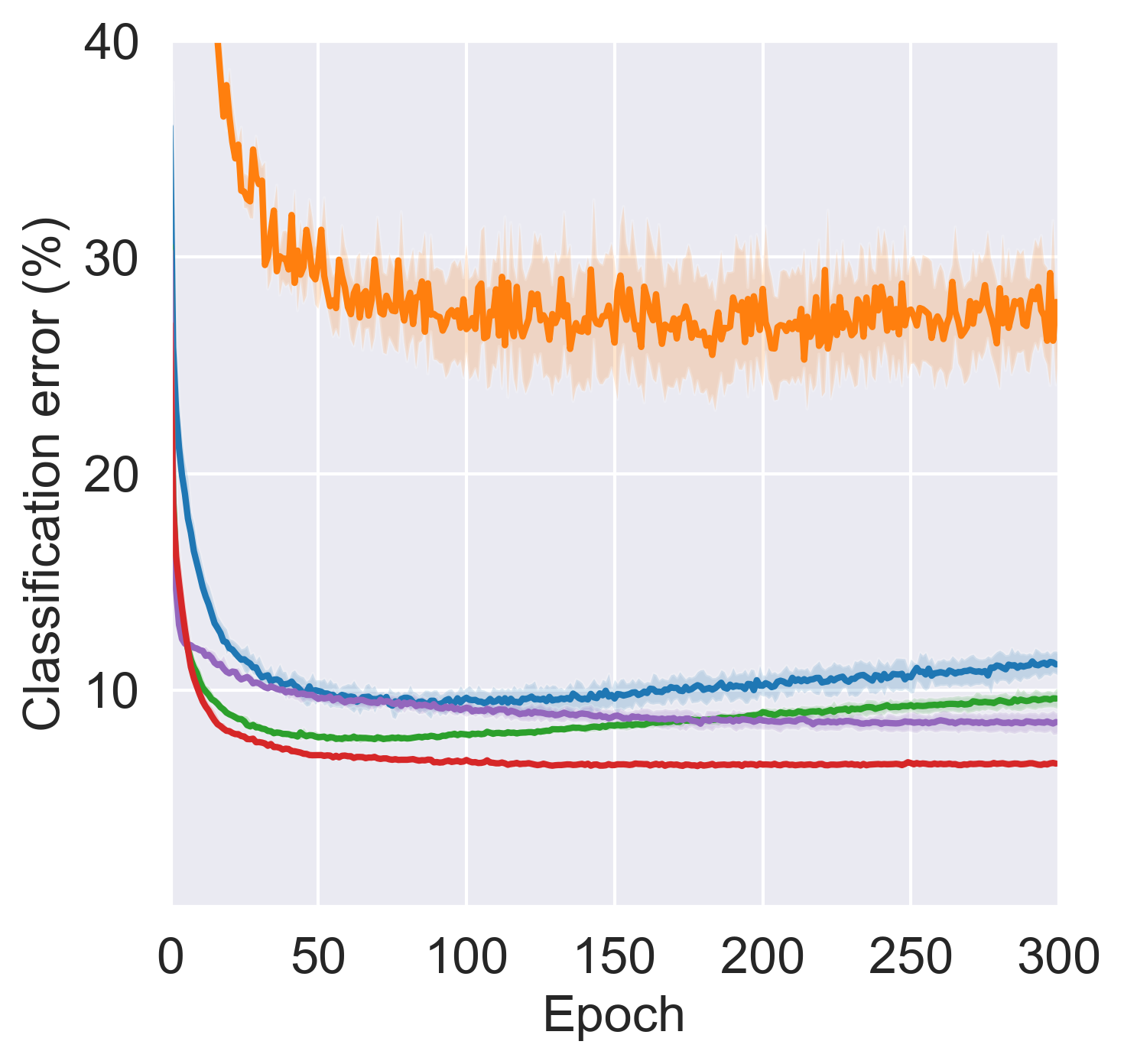}
        \includegraphics[width=0.33\textwidth]{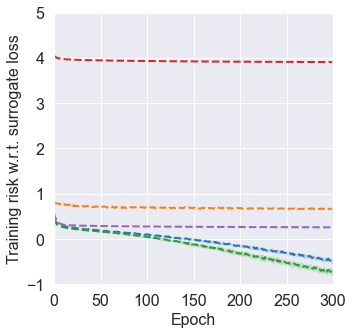}
    \end{minipage}\\
    \begin{minipage}[c]{0.08\textwidth}\flushright\small Kuzushiji\\-MNIST \end{minipage}%
    \hspace{0.5em}
    \begin{minipage}[c]{0.90\textwidth}
        \includegraphics[width=0.33\textwidth]{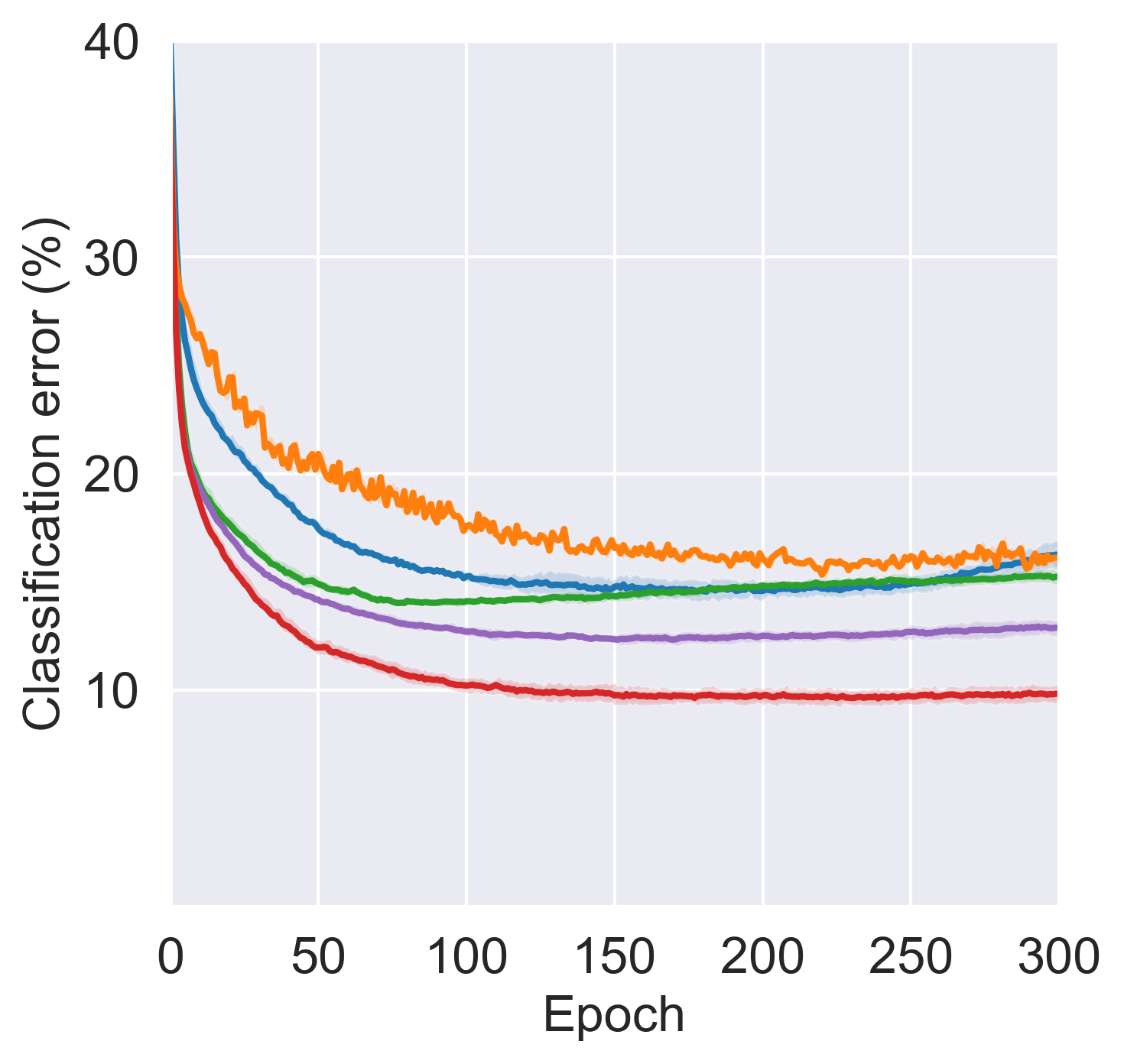}%
        \includegraphics[width=0.33\textwidth]{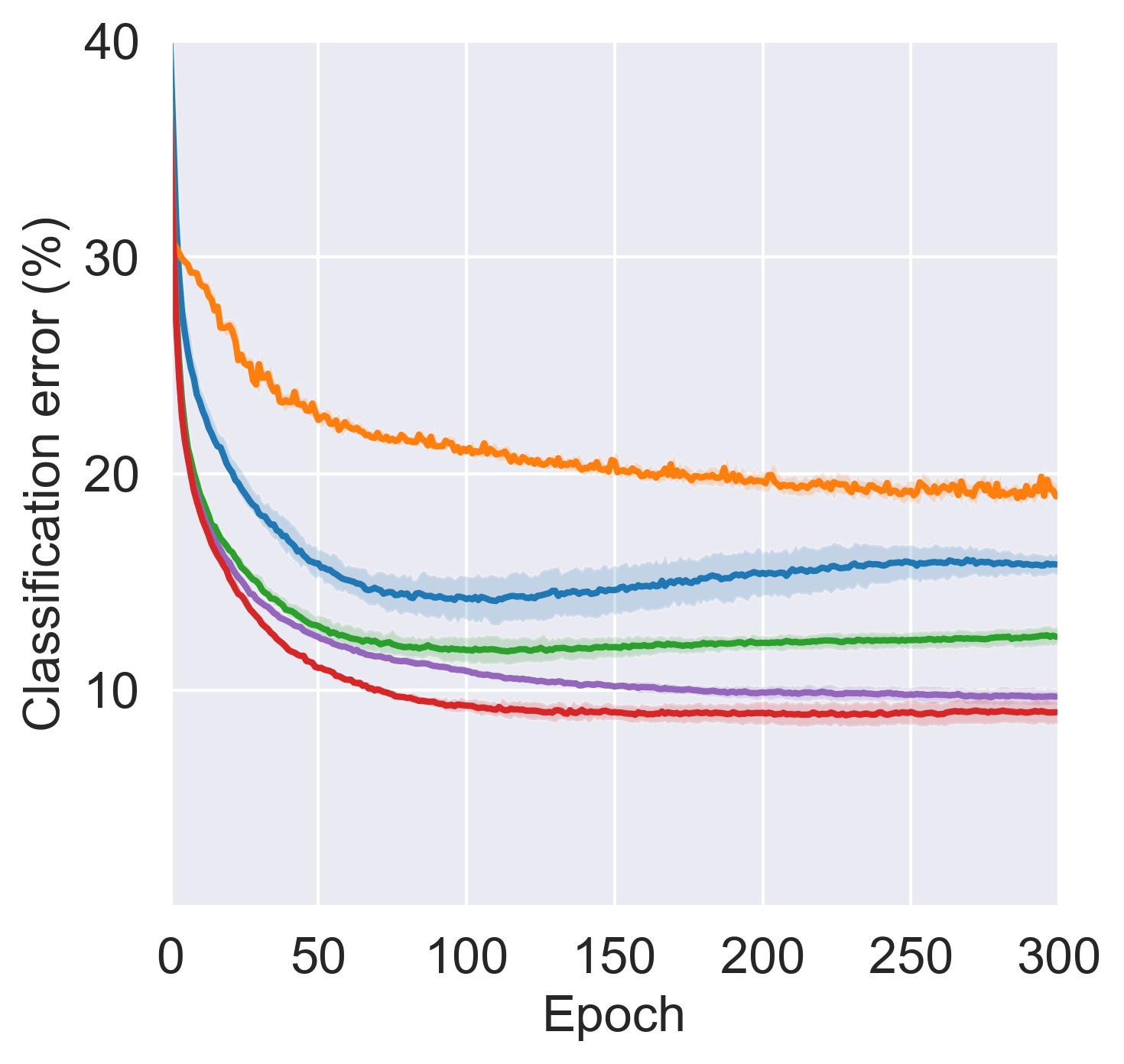}
        \includegraphics[width=0.33\textwidth]{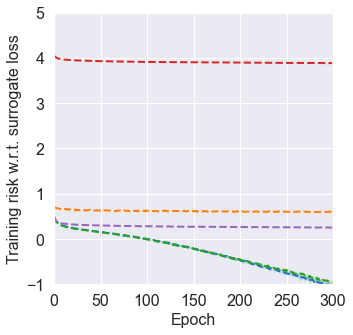}
    \end{minipage}\\
    \begin{minipage}[c]{0.08\textwidth}\flushright\small CIFAR-10 \end{minipage}%
    \hspace{0.5em}
    \begin{minipage}[c]{0.90\textwidth}
        \includegraphics[width=0.33\textwidth]{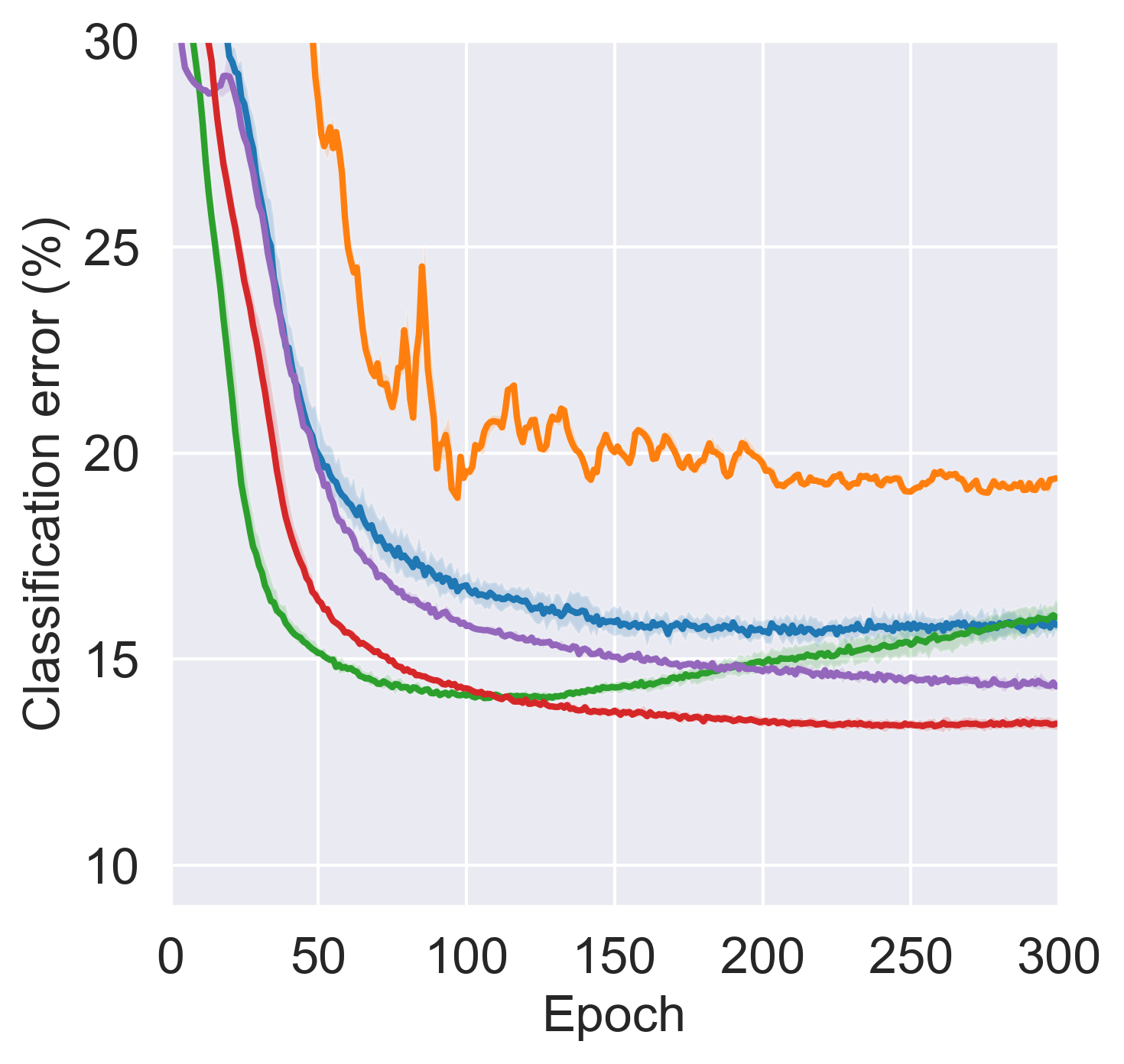}%
        \includegraphics[width=0.33\textwidth]{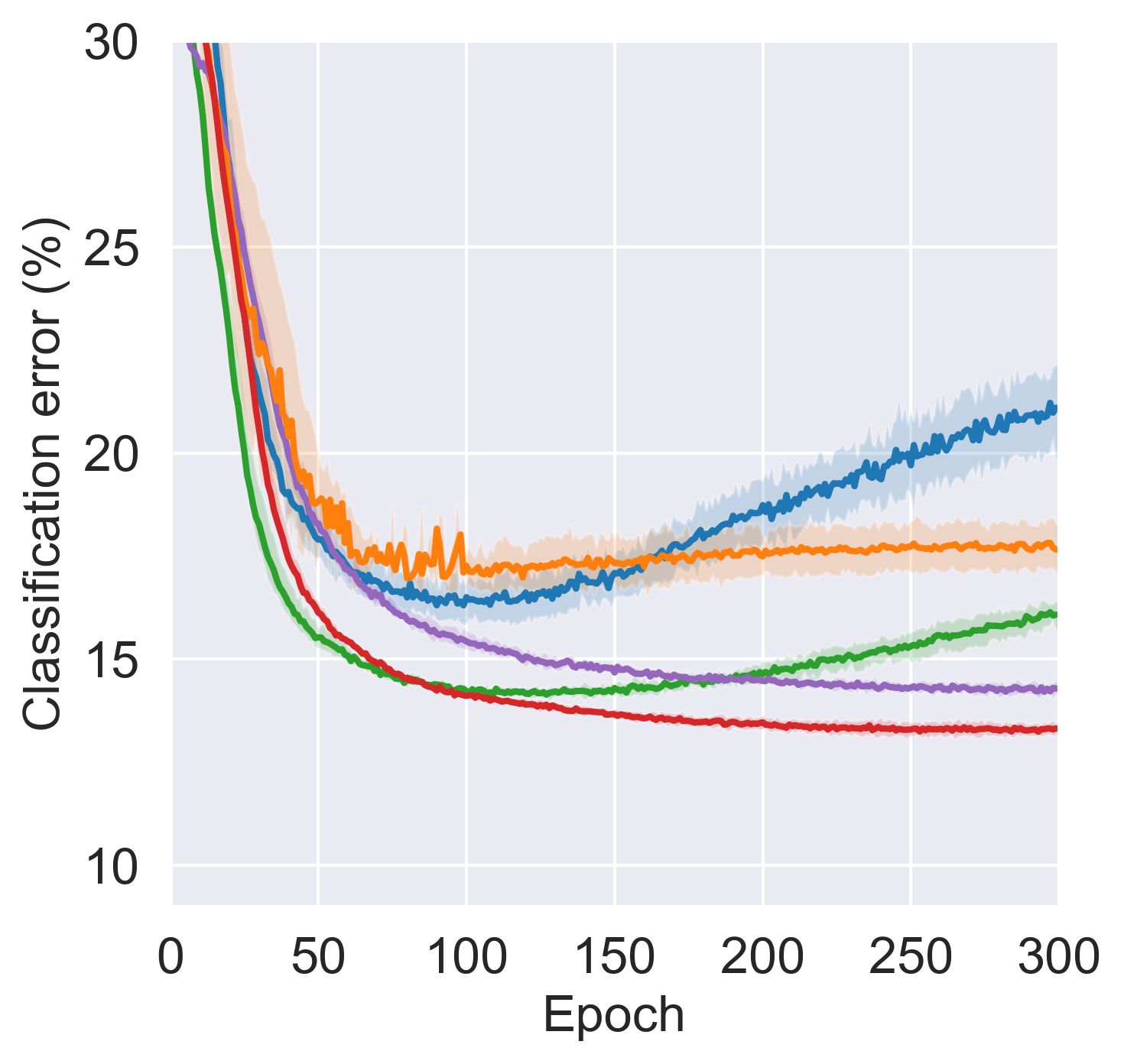}
        \includegraphics[width=0.33\textwidth]{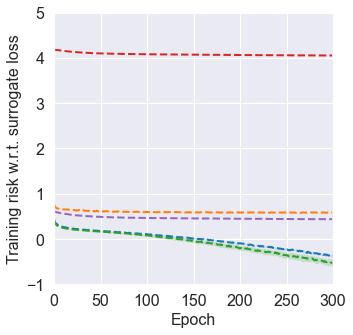}
    \end{minipage}
    
    \caption{Experimental results of learning from 10 and 50 sets of U data.
    Solid curves are the test errors (in percentage) and dashed curves are the empirical training risks.
    Dark colors show the mean errors (risks) of 3 trials and light colors show the standard deviations.}
    \label{fig:performance}
\end{figure*}

\begin{figure*}[t]
    \centering
    \begin{minipage}[c]{\textwidth}
        \includegraphics[width=0.25\textwidth]{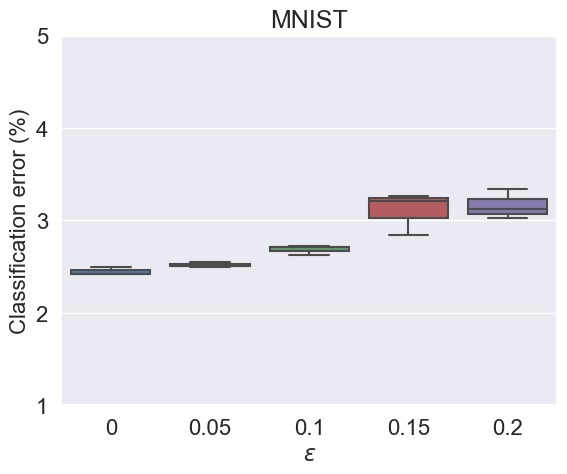}%
        \includegraphics[width=0.25\textwidth]{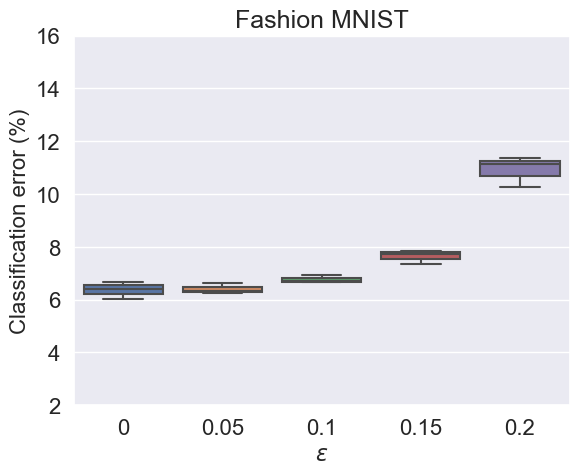}%
        \includegraphics[width=0.25\textwidth]{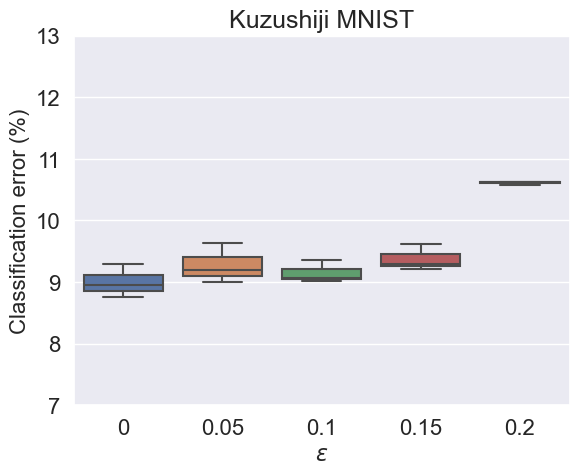}%
        \includegraphics[width=0.25\textwidth]{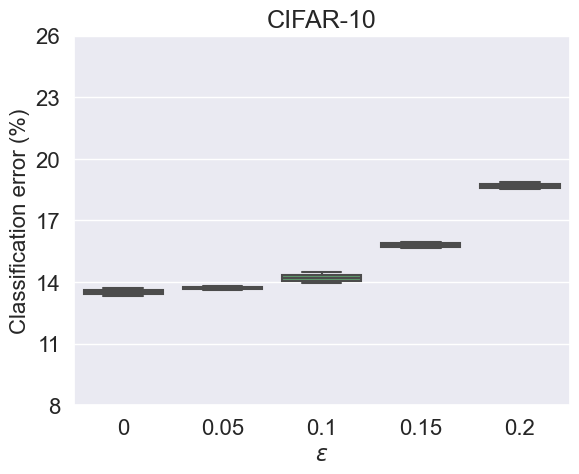}%
    \end{minipage}
    \vspace{-1em}%
    \caption{Box plot of the classification errors for the proposed U$^m$-SSC method tested on learning from 50 U sets with inaccurate class priors ($\epsilon=0$ means \emph{true}; larger $\epsilon$, larger noise).}
    \label{fig:robust}
    \vspace{-1em}%
\end{figure*}

    \begin{table*}[t]
        \caption{Mean errors (standard deviations) over 3 trials in percentage for the proposed U$^m$-SSC method tested on different set sizes. The uniform set size $n_j$ is shifted to $\tau\cdot n_j$ (smaller $\tau$, larger shift). Random means uniformly sample a set size from range $[0,T]$.}
        \label{size_shift_table}
        \begin{center}
        \newcommand{\tabincell}[2]{\begin{tabular}{@{}#1@{}}#2\end{tabular}}
        \begin{tabular}{ c| c | c | c c c c | c} 
        \hline
        Dataset & \tabincell{c}{Sets} & \tabincell{c}{$n_j$} & \tabincell{c}{$\tau=0.8$} & \tabincell{c}{$\tau=0.6$} & \tabincell{c}{$\tau=0.4$} & \tabincell{c}{\textbf{$\tau=0.2$}} & \tabincell{c}{{Random}}\\ 
        \hline
        \multirowcell{2}{MNIST}      
        & 10 & 6000 & 2.83 (0.18) & 2.91 (0.04) & 3.2 (0.2) & 3.19 (0.35) & 2.66 (0.08)\\
		& 50 & 1200 & 2.46 (0.1) & 2.58 (0.08) & 2.76 (0.13) & 2.97 (0.11) & 3.0 (0.12)\\
		\hline
		\multirowcell{2}{Fashion-MNIST}      
        & 10 & 6000 & 7.88 (0.21) & 7.68 (0.36) & 7.84(0.17) & 7.89 (0.24) & 7.04 (0.13)\\
		& 50 & 1200 & 8.29 (0.19) & 8.91 (0.29) & 7.61 (0.55) & 8.8 (0.31) & 8.62 (0.24)\\
		\hline
		\multirowcell{2}{Kuzushiji-MNIST}      
        & 10 & 6000 & 8.98 (0.52) & 9.83 (0.57) & 9.43 (0.41) & 10.03 (0.81) & 8.38 (0.31)\\
		& 50 & 1200 & 9.35 (0.33) & 9.53 (0.64) & 9.89 (0.72) & 11.08 (0.61) & 10.34 (0.73)\\
		\hline
		\multirowcell{2}{CIFAR-10}      
        & 10 & 5000 & 12.55 (0.61) & 12.25 (0.8) & 12.41 (0.35) & 12.49 (0.98) & 11.65 (0.44)\\
		& 50 & 1000 & 12.16 (0.23) & 12.19 (0.75) & 12.88 (0.37) & 13.66 (0.54) & 12.09 (0.42)\\
        \hline
        \end{tabular}
        \end{center}
        \vspace{-0.5em}
    \end{table*} 

\section{Experiments}
\label{sec:exp}%
In this section, we experimentally analyze the proposed method and compare it with state-of-the-art methods in the U$^m$ classification setting.%
\footnote{Our implementation of U$^m$-SSC is available at \url{https://github.com/leishida/Um-Classification}.}

\paragraph{Datasets}
We train on widely adopted benchmarks MNIST, Fashion-MNIST, Kuzushiji-MNIST, and CIFAR-10.
Table~\ref{dataset} briefly summarizes the benchmark datasets.
Since the four datasets contain 10 classes originally, we manually corrupt them into binary classification datasets.
More details about the datasets are in Appendix~\ref{setup-datasets}.

In the experiments, unless otherwise specified,
the number of training data contained in all U sets are the same and fixed as $n_j=n_{\rm{tr}}/m$ for all benchmark datasets;
and the class priors $\{\pi_j\}_{j=1}^m$ of all U sets are randomly sampled from the range $[0.1, 0.9]$ under the constraint that the sampled class priors are not all identical, ensuring that the problem is mathematically solvable.
Given $\{n_j\}_{j=1}^m$ and $\{\pi_j\}_{j=1}^m$, we generate $m$ sets of U training data following \eqref{kusets}.
Note that in most LLP papers, each U set is uniformly sampled from the shuffled U training data, therefore the label proportions of all the U sets are the same in expectation.
As the set size increases, all the proportions converge to the same class prior, making the LLP problem computationally intractable \cite{scott2020learning}.
As shown above, our experimental scheme avoids this issue by determining valid class priors before sampling each U set.

\paragraph{Models}
The models are optimizers used are also described in Table~\ref{dataset}, where MLP refers to \emph{multi-layer perceptron}, ResNet refers to \emph{residual networks} \citep{he16cvpr}, and their detailed architectures are in Appendix~\ref{setup-models}.
As a common practice, we use Adam \citep{kingma15iclr} with the cross-entropy loss for optimization. 
We train 300 epochs for all the experiments, and the classification error rates at the test phase are reported. 
All the experiments are repeated 3 times and the mean values with standard deviations are recorded for each method.

\paragraph{Baselines}
We compare the proposed method with state-of-the-art methods based on the classification risk \cite{scott2020learning} and the empirical proportion risk \cite{tsai2020learning} for the U$^m$ classification problem.
Recall that the proposed learning objective $\widehat{R}_{\rm{U^m}}(f)$ in \citet{scott2020learning} is a combination of the unbiased balanced risk estimators $\widehat{R}_{\rm{U^2}\text{-}\rm{b}}(f)$, which are shown to underperform the unbiased risk estimator $\widehat{R}_{\rm{U^2}}(f)$ in \citet{lu2018minimal}.
So we improve the baseline method of \citet{scott2020learning} by combining $\widehat{R}_{\rm{U^2}}(f)$ instead of $\widehat{R}_{\rm{U^2}\text{-}\rm{b}}(f)$.
As shown in \citet{lu2020mitigating}, the empirical risks $\widehat{R}_{\rm{U^2}}(f)$ can go negative during training which may cause overfitting, so we further improve the baseline by combining the corrected non-negative risk estimators $\widehat{R}_{\rm{U^2}\text{-}\rm{c}}(f)$.
The baselines are summarized as follows:
    \begin{itemize}
        \item MMC$\text{-}{\rm{U^2}\text{-}\rm{b}}$ \cite{scott2020learning}: the classification risk based method in the \emph{multiple mutual contamination}~(MMC) framework, i.e., \eqref{muub};
        \item MMC$\text{-}{\rm{U^2}}$: the method that improves MMC$\text{-}{\rm{U^2}\text{-}\rm{b}}$ with unbiased risk estimators $\widehat{R}_{\rm{U^2}}(f)$;
        \item MMC$\text{-}{\rm{U^2}\text{-}\rm{c}}$: the method that improves MMC$\text{-}{\rm{U^2}}$ with non-negative risk correction $\widehat{R}_{\rm{U^2}\text{-}\rm{c}}(f)$;
        \item LLP-VAT \cite{tsai2020learning}: the empirical proportion risk based method, i.e., \eqref{llpvat}.
    \end{itemize}
More details about the implementation of baselines can be found in Appendix~\ref{setup-otherdetail}.

\subsection{Comparison with State-of-the-art Methods}
\label{exp-sota}
We first compare our proposed method with state-of-the-art methods for the U$^m$ classification problem.
The experimental results of learning from 10 and 50 U sets are reported in Figure~\ref{fig:performance} and a table of the final errors is in Appendix~\ref{negativeloss}.

We can see that the \emph{classification risk} based methods, i.e., MMC$\text{-}{\rm{U^2}\text{-}\rm{b}}$, MMC$\text{-}{\rm{U^2}}$, MMC$\text{-}{\rm{U^2}\text{-}\rm{c}}$, and our proposed U$^m$-SSC method generally outperform the \emph{empirical proportion risk} based method, i.e., LLP-VAT, with lower classification error and more stability,
which demonstrates the superiority of the \emph{consistent} methods.

Within the classification risk based methods, our observations are as follows.
First, the proposed U$^m$-SSC method outperforms others in most cases.
We believe that the advantage comes from the surrogate set classification mechanism in U$^m$-SSC, which implies the \emph{classifier-consistent} methods perform better than the \emph{risk-consistent} methods.
Second, compared to MMC$\text{-}{\rm{U^2}\text{-}\rm{b}}$, we can see our advantage becomes bigger when $\pi_{\mathcal{D}}\approx\frac{1}{2}$ is not the case, e.g., Fashion-MNIST and CIFAR-10.
Moreover, the performance of the improved MMC$\text{-}{\rm{U^2}}$ method (combination of unbiased risk estimators) is better than MMC$\text{-}{\rm{U^2}\text{-}\rm{b}}$ (combination of balanced risk estimators) in all cases.
These empirical findings corroborate our analysis that the balanced classification risk \eqref{balrisk} can be biased in such cases.
Third, we confirm that the training risks of MMC$\text{-}{\rm{U^2}\text{-}\rm{b}}$ and MMC$\text{-}{\rm{U^2}}$ go negative as training proceeds, which incurs overfitting.
Other methods do not have this negative empirical training risk issue.
And we can see that the improved MMC$\text{-}{\rm{U^2}\text{-}\rm{c}}$ method effectively mitigates this overfitting but its performance is still inferior to our proposed method.
These results are consistent with the observations in \citet{lu2020mitigating}.
We also notice that the empirical training risks of the proposed U$^m$-SSC method are obviously higher than other baseline methods. 
This is due to the fact that the added transition layer rescales the range of model output.
We provide a detailed analysis on this point in Appendix~\ref{negativeloss}.
A notable effect is that a relatively small learning rate is more suitable for our method.

\subsection{On the Variation of Set Size}
In practice, the size of the U sets may vary from a large range depends on different tasks.
However, as the set size varies, given the data generation process in \eqref{ptr}, the marginal density of our training data $p_{\rm{tr}}(\boldsymbol{x})$ shifts from that of the test one, which may cause severe \emph{covariate shift} \citep{shimodaira2000improving,zhang2020one}.
To verify the robustness of our proposed method against covariate shift, we conducted experiments on the variation of set size. 
Recall that in other experiments, we use uniform set size, i.e., all sets contain $n_{\rm{tr}}/m$ U data. 
In this subsection, we investigate two kinds of set size shift:
    \begin{itemize}
        \item Randomly select $\lceil m/2\rceil$ U sets and change their set sizes to $\tau\cdot n_{\rm{tr}}/m$ where $\tau\in [0, 1]$;
        \item Randomly sample each set size $n_j$ from range $[0,n_{\rm{tr}}]$ such that $\sum_{j = 1}^m n_j=n_{\rm{tr}}$.
    \end{itemize}
As shown in Table~\ref{size_shift_table}, the proposed method is reasonably robust as $\tau$ moves towards 0 in the first shift setting. 
The slight performance degradation may come from the decreased total number of training samples $n_{\rm{tr}}$ as $\tau$ decreases.
We also find that our method reaches the best performance in 3 out of 4 benchmark datasets in the second shift setting. 
Since it is a more natural way for generating set sizes, the robustness of the proposed method on varied set sizes can be verified.

\begin{table}[t]
\centering
\caption{Mean test errors (standard deviations) over 3 trials in percentage for the U$^m$-SSC method tested with different set numbers.}
\newcommand{\tabincell}[2]{\begin{tabular}{@{}#1@{}}#2\end{tabular}}
\begin{center}
\begin{tabular}{ c| c c c c} 
\hline
Dataset & \tabincell{c}{2 sets} & \tabincell{c}{100 sets} & \tabincell{c}{500 sets} & \tabincell{c}{1000 sets}\\
\hline
\multirow{2}{*}{MNIST}      
    & 2.36 & 2.59 & 3.07 & 2.84\\
    & (0.27) & (0.11) & (0.14) & (0.16)\\
\hline
\multirow{2}{*}{\tabincell{c}{Fashion-\\MNIST}}      
& 7.95 & 8.61 & 8.50 & 8.64\\
& (0.85) & (0.43) & (0.44) & (0.26)\\
\hline
\multirow{2}{*}{\tabincell{c}{Kuzushiji-\\MNIST}}   
& 9.68 & 10.20 & 11.64 & 10.68\\
& (0.96) & (0.50) & (0.54) & (0.40)\\
\hline
\multirow{2}{*}{CIFAR-10}      
& 13.21 & 13.30 & 12.98 & 13.16\\
& (1.26) & (0.65) & (0.37) & (0.70)\\
\hline
\end{tabular}
\end{center}
\label{sets_shift_table}
\vspace{-1.5em}
\end{table}

\subsection{On the Variation of Set Number}
\label{setnum}
Another main factor that may affect the performance is the number of available U sets.
As the U data can be easily collected from multiple sources, the learning algorithm is expected to be able to handle the variation of set numbers well.
The experimental results of learning from 10 and 50 U sets have been shown in Section~\ref{exp-sota}. 
In this subsection, we test the proposed U$^m$-SSC method on extremely small set numbers e.g., $m$ = 2, and large set numbers, e.g., $m$=1000.
The experimental results of learning form 2, 100, 500, and 1000 U sets are reported in Table \ref{sets_shift_table}. 

From the results, we can see that the performance of the proposed method is reasonably well on different set numbers.
In particular, a lower classification error can be observed for $m=2$ across all 4 benchmark datasets. 
The better performance may come from the larger number of U data contained in a single set, i.e., $n_{\rm{tr}}/2$ in this case.
Since our method uses class priors as the only weak supervision, an increasing number of the sampled data within each U set guarantees a better approximation of them.
These experimental results demonstrate the effectiveness of the proposed method on the variation of set numbers.
We note that it is a clear advantage over the LLP methods, whose performance drops significantly when the set number becomes small and set size becomes large, because label proportions converge to the same class prior in their setup, making the LLP problem computationally intractable.

\subsection{Robustness against Inaccurate Class Priors}
Hitherto, we have assumed that the values of class priors $\{\pi_j\}_{j=1}^{m}$ are accessible and accurately used in the construction of our method, which may not be true in practice. 
In order to simulate U$^m$ classification in the wild, where we may suffer some errors from estimating the class priors, we design experiments that add noise to the true class priors.
More specifically, we test the U$^m$-SSC method by replacing $\pi_j$ with the noisy
$\pi_j'=\pi_j+\gamma\cdot\epsilon$, where $\gamma$ uniform randomly take values in $\{+1,-1\}$ and $\epsilon\in\{0, 0.05,0.1,0.15,0.2\}$, so that the method would treat noisy $\pi_j'$ as the true $\pi_j$ during the whole learning process.
The experimental setup is exactly same as before except the replacement of $\pi_j$.
Note that we tailor the noisy $\pi_j$ to $[0, 1]$ if it surpasses the range.

The results on learning from 50 U sets with inaccurate class priors are reported in Figure~\ref{fig:robust} and a table of the final test errors is in Appendix~\ref{inaccurate}, where $\epsilon=0$ means true class priors.
We can see that our method works reasonably well using noisy $\pi_j'$, though the classification error slightly increases for higher noise level $\epsilon$ which is as expected.

\section{Conclusions}
\label{sec:concl}%
In this work, we focused on learning from multiple sets of U data and proposed a new method based on a surrogate set classification task.
We bridged the original and surrogate class-posterior probabilities via a linear-fractional transformation, and then studied its properties.
Based on them, we proposed the U$^m$-SSC algorithm and implemented it by adding a transition layer to the neural network.
We also proved that the U$^m$-SSC method is classifier-consistent and established an estimation error bound for it.
Extensive experiments demonstrated that the proposed method could successfully train binary classifiers from multiple U sets, and it compared favorably with state-of-the-art methods.

\subsubsection*{Acknowledgments}
The authors would like to thank Tianyi Zhang, Wenkai Xu, Nontawat Charoenphakdee, and Yuting Tang for helpful discussions. 
NL was supported by the MEXT scholarship No.\ 171536 and MSRA D-CORE Program.
GN and MS were supported by JST AIP Acceleration Research Grant Number
JPMJCR20U3, Japan.
MS was also supported by the Institute for AI and Beyond, UTokyo.
\clearpage

\bibliography{references}
\bibliographystyle{icml2021}

\clearpage
\onecolumn
\begin{center}
\LARGE Supplementary Material
\end{center}{}

\appendix
\section{Proofs}
\label{sec:proof}%

In this appendix, we prove all theorems.

\subsection{Proof of Theorem~\ref{thm:transformation}}
\label{subsec:thm1}%
On one hand, $\forall j \in [1,\ldots,m]$ we have
\begin{align}
\bar{\eta}_j(\boldsymbol{x})&=\frac{p(\boldsymbol{x}, \bar{y}=j)}{\bar{p}(\boldsymbol{x})}\notag\\
&=\frac{p(\boldsymbol{x}\mid \bar{y}=j)\cdot p(\bar{y}=j)}{\bar{p}(\boldsymbol{x})}\notag\\
&=\frac{\rho_j\cdot[\pi_j\cdot p_p(\boldsymbol{x})+(1-\pi_j)\cdot p_n(\boldsymbol{x})]}{\sum^m_{j=1}\rho_j\cdot[\pi_j\cdot p_p(\boldsymbol{x})+(1-\pi_j)\cdot p_n(\boldsymbol{x})]}. 
\label{proof1}
\end{align}
In the third equality, we substitute $p(\boldsymbol{x}\mid \bar{y}=j)$ with $p_{\rm{tr}}$ that is defined in \eqref{ptr}. 
On the other hand, by Bayes' rule we have
\begin{align}
\label{proof2-1}
&p_p(\boldsymbol{x})=p(\boldsymbol{x}\mid y = +1) = \frac{p(y=+1\mid \boldsymbol{x})\cdot p(\boldsymbol{x})}{p(y=+1)}=\frac{\eta(\boldsymbol{x})\cdot p(\boldsymbol{x})}{\pi_{\mathcal{D}}},\\
&p_n(\boldsymbol{x})=p(\boldsymbol{x}\mid y = -1) = \frac{p(y=-1\mid \boldsymbol{x})\cdot p(\boldsymbol{x})}{p(y=-1)}=\frac{(1-\eta(\boldsymbol{x}))\cdot p(\boldsymbol{x})}{1-\pi_{\mathcal{D}}}. \label{proof2}
\end{align}
Then, we plug \eqref{proof2-1} and \eqref{proof2} into \eqref{proof1} and obtain
\begin{align*}
\bar{\eta}_j(\boldsymbol{x})&=\frac{\rho_j\cdot[\pi_j \eta(\boldsymbol{x})  \cdot(1 - \pi_{\mathcal{D}})+(1-\pi_j)\cdot (1-\eta(\boldsymbol{x}))\cdot \pi_{\mathcal{D}}]}{\sum^m_{j=1}\rho_j\cdot[\pi_j \eta(\boldsymbol{x}) \cdot(1 - \pi_{\mathcal{D}})+(1-\pi_j)\cdot (1-\eta(\boldsymbol{x}))\cdot \pi_{\mathcal{D}}]}\\
&=\frac{\rho_j\cdot(\pi_j - \pi_{\mathcal{D}})\cdot\eta(\boldsymbol{x})+\rho_j\cdot(1-\pi_j)\cdot\pi_{\mathcal{D}}}{\sum^m_{j=1}\rho_j\cdot(\pi_j - \pi_{\mathcal{D}})\cdot\eta(\boldsymbol{x})+\sum^m_{j=1}\rho_j\cdot(1-\pi_j)\cdot\pi_{\mathcal{D}}}.
\end{align*}
By setting the coefficients $a_j$, $b_j$, $c$, $d$ accordingly we conclude the proof.
\qed

\subsection{Proof of Lemma~\ref{lemma_injective}}
\label{subsec:lemma_jnjective}%
We proceed the proof by firstly showing that the denominator of each function $T_j(t)$, $j={1,\ldots,m}$, is strictly greater than zero for all $t \in [0,1]$, and then showing that $\boldsymbol{T}(t_1)=\boldsymbol{T}(t_2)$ if and only if $t_1 = t_2$.

For all $j={1,\ldots,m}$, the denominators of $T_j(t)$ are the same, i.e., $c\cdot t+d$, where $c = \sum_{j=1}^m{\rho}_j(\pi_j-\pi_{\mathcal{D}})$ and $d = \sum_{j=1}^m{\rho}_j\pi_{\mathcal{D}}(1-\pi_j)$. We know that $d$ is positive because $\rho_j > 0$, $\pi_\mathcal{D} > 0$, and there exists $j \in {1,\ldots,m}$ such that $\pi_j<1$.
       Given that $t \in [0, 1]$, we discuss the sign of $c$ as follows:
       \begin{enumerate}
            \item if $c\geq0$, the minimum value of $c\cdot t + d$ is $c\cdot 0 + d=d>0$;
            \item if $c<0$, the minimum value of $c\cdot t + d$ is $c\cdot 1 + d=\sum_{j=1}^m{\rho}_j(\pi_j-\pi_{\mathcal{D}})+\sum_{j=1}^m{\rho}_j\pi_{\mathcal{D}}(1-\pi_j)=\sum_{j=1}^m{\rho}_j\pi_j(1-\pi_{\mathcal{D}})>0$, where the last inequality is due to the existence of $j \in {1,\ldots,m}$ such that $\pi_j>0$.
       \end{enumerate}
       Hitherto, we have shown that the denominator $c\cdot t + d>0$. 
       Next, we prove the one-to-one mapping property by contradiction. 
       Assume that there exist $t_1, t_2 \in [0, 1]$ such that $t_1\neq t_2$ but $\boldsymbol{T}(t_1)=\boldsymbol{T}(t_2)$, which indicates that $T_j(t_1)=T_j(t_2), \forall j = {1,\ldots,m}$. For all $j$, we have
       \begin{align}
           T_j(t_1)-T_j(t_2)&=\frac{a_j\cdot t_1+b_j}{c\cdot t_1+d}-\frac{a_j\cdot t_2+b_j}{c\cdot t_2+d}\notag\\
           &=\frac{(a_j\cdot t_1+b_j)((c\cdot t_2+d))-(a_j\cdot t_2+b_j)((c\cdot t_1+d))}{(c\cdot t_1+d)(c\cdot t_2+d)}\notag\\
           \label{proof_jnjective}
           &=\frac{(t_1-t_2)(a_j\cdot d-b_j\cdot c)}{(c\cdot t_1+d)(c\cdot t_2+d)}\\
           &=0\notag,
       \end{align}
       where $a_j = \rho_j\cdot(\pi_j - \pi_{\mathcal{D}})$ and $b_j = \rho_j\cdot(1 - \pi_j)\cdot\pi_{\mathcal{D}}$. 
       As shown previously, the denominator of \eqref{proof_jnjective} is non-zero for all $j$.
       Next we show that there exists $j \in {1,\ldots,m}$ such that $a_j\cdot d-b_j c\neq0$.
       Since $c$ and $d$ are constants and irrelevant to $i$, we have 
       \begin{equation}\nonumber
       \begin{split}
           a_j\cdot d - b_j \cdot c &= (\rho_j\cdot(\pi_j - \pi_{\mathcal{D}}))\cdot d-(\rho_j\cdot(1 - \pi_j)\cdot\pi_{\mathcal{D}})\cdot c\\
           &=\rho_j\cdot (\pi_j\cdot d-\pi_{\mathcal{D}}\cdot d-c+\pi_j\cdot c)\\
           &=\rho_j\cdot (\pi_j\cdot (c + d)-\pi_{\mathcal{D}}\cdot d-c).
       \end{split}
       \end{equation}
        This equation equals to zero if and only if $\pi_j = \frac{c + \pi_{\mathcal{D}}\cdot d}{c+d}$.
        According to our assumption that at least two of the U sets are different, $\exists j' \in {1,\ldots,m}$ such that $\pi_{j'} \neq \frac{c + \pi_{\mathcal{D}}\cdot d}{c+d}$. 
        For such $j'$, $T_{j'}(t_1)=T_{j'}(t_2)$ if and only if $t_1=t_2$, which leads to a contradiction since $t_1\neq t_2$.
        So we conclude the proof that $\boldsymbol{T}(t_1)=\boldsymbol{T}(t_2)$ if and only if $t_1 = t_2$.
\qed

\subsection{Proof of Lemma~\ref{lemma_g_star}}
\label{subsec:lemma_g_star}%
We provide a proof of the cross-entropy loss and mean squared error, which are commonly used losses because of their numerical stability and good convergence rate \cite{de2005tutorial, allen1971mean}.

\paragraph{Cross-entropy loss} Since the cross-entropy loss is non-negative by its definition, minimizing $R_{\rm{surr}}(\boldsymbol{g})$ can be obtained by minimizing the conditional risk $\mathbb{E}_{p(\bar{y}\mid \boldsymbol{x})}[\ell(\boldsymbol{g}(\boldsymbol{x}),\bar{y})\mid \boldsymbol{x}]$ for every $\boldsymbol{x} \in \mathcal{X}$. So we are now optimizing
        \begin{equation} \nonumber
            \begin{gathered}
                \phi(\boldsymbol{g})=-\sum_{j=1}^{m}p(\bar{y}=j\mid \boldsymbol{x})\cdot \log(g_j(\boldsymbol{x})), \quad {\rm{s.t.}} \sum_{j=1}^{m}g_j(\boldsymbol{x})=1.
            \end{gathered}
        \end{equation}
        By using the Lagrange multiplier method \cite{bertsekas1997nonlinear}, we have
        \begin{equation} \nonumber
            \mathcal{L}=-\sum_{j=1}^{m}p(\bar{y}=j\mid \boldsymbol{x})\cdot \log(g_j(\boldsymbol{x}))-\lambda\cdot(\sum_{j=1}^{m}g_j(\boldsymbol{x})-1).
        \end{equation}
        The derivative of $\mathcal{L}$ with respect to $\boldsymbol{g}$ is
        \begin{equation} \nonumber
            \frac{\partial\mathcal{L}}{\partial\boldsymbol{g}} = [-\frac{p(\bar{y}=1\mid \boldsymbol{x})}{g_1(\boldsymbol{x})}-\lambda,\cdot\cdot\cdot,-\frac{p(\bar{y}=m\mid \boldsymbol{x})}{g_m(\boldsymbol{x})}-\lambda]^{\top}.
        \end{equation}
        By setting this derivative to 0 we obtain
        \begin{equation} \nonumber
            g_j(\boldsymbol{x})=-\frac{1}{\lambda}\cdot p(\bar{y}=j\mid \boldsymbol{x}), \quad \forall j=1,\ldots,m \ \ {\rm{and}}\ \  \forall \boldsymbol{x} \in \mathcal{X}.    
        \end{equation}
        Since $\boldsymbol{g} \in \Delta^{m-1}$ is the $m$-dimensional simplex, we have $\sum_{j=1}^{m}g_j^{\star}(\boldsymbol{x})=1$ and $\sum_{j=1}^{m}p(\bar{y}=j\mid \boldsymbol{x})=1$.
        Then
        \begin{equation} \nonumber
            \sum_{j=1}^{m}g_j^{\star}(\boldsymbol{x})=-\frac{1}{\lambda}\cdot\sum_{j=1}^{m}p(\bar{y}=j\mid \boldsymbol{x})=1.
        \end{equation}
        Therefore we obtain $\lambda=-1$ and $g_j^{\star}(\boldsymbol{x})=p(\bar{y}=j\mid \boldsymbol{x})=\bar{\eta}_j(\boldsymbol{x}),\forall j = 1,\ldots,m$ and $\forall \boldsymbol{x} \in \mathcal{X}$, which is equivalent to $\boldsymbol{g}^{\star} = \bar{\boldsymbol{\eta}}$. Note that when $m=2$, the softmax is reduce to sigmoid function and the cross-entropy is reduced to logistic loss $\ell_{\rm{log}}(z)=\ln(1+\exp(-z))$.
        
        \paragraph{Mean squared error} 
        Similarly to the cross-entropy loss, we transform the risk minimization problem to the following constrained optimization problem
        \begin{equation} \nonumber
            \begin{gathered}
                \phi(\boldsymbol{g})=\sum_{j=1}^{m}(p(\bar{y}=j\mid \boldsymbol{x})-g_j(\boldsymbol{x}))^2, \quad {\rm{s.t.}} \sum_{j=1}^{m}g_j(\boldsymbol{x})=1.
            \end{gathered}
        \end{equation}
        By using the Lagrange multiplier method, we obtain
        \begin{equation} \nonumber
            \mathcal{L}=\sum_{j=1}^{m}(p(\bar{y}=j\mid \boldsymbol{x})-g_j(\boldsymbol{x}))^2-\lambda\cdot(\sum_{j=1}^{m}g_j(\boldsymbol{x})-1).
        \end{equation}
        The derivative of $\mathcal{L}$ with respect to $\boldsymbol{g}$ is
        \begin{equation} \nonumber
            \frac{\partial\mathcal{L}}{\partial\boldsymbol{g}} = [2g_1(\boldsymbol{x})-2p(\bar{y}=1\mid \boldsymbol{x})-\lambda,\cdot\cdot\cdot,2g_m(\boldsymbol{x})-2p(\bar{y}=m\mid \boldsymbol{x})-\lambda]^{\top}.
        \end{equation}
        By setting this derivative to 0 we obtain
        \begin{equation} \nonumber
            g_j(\boldsymbol{x})=p(\bar{y}=j\mid \boldsymbol{x})+\frac{\lambda}{2}.
        \end{equation}
        Since $\sum_{j=1}^{m}g_j^{\star}(\boldsymbol{x})=1$ and $\sum_{j=1}^{m}p(\bar{y}=j\mid \boldsymbol{x})=1$, we have
        \begin{align*}
        \sum_{j=1}^{m}g_j^{\star}(\boldsymbol{x})&=\sum_{j=1}^{m}p(\bar{y}=j\mid \boldsymbol{x})+\frac{\lambda\cdot m}{2},\\
        \frac{\lambda\cdot m}{2}&=0.
        \end{align*}
        Since $m \geq 2$, we can obtain $\lambda=0$. Consequently, $g_j^{\star}(\boldsymbol{x})=p(\bar{y}=j\mid \boldsymbol{x})=\bar{\eta}_j(\boldsymbol{x})$, which leads to $\boldsymbol{g}^{\star}=\bar{\boldsymbol{\eta}}$.
        We conclude the proof.
\qed

\subsection{Proof of Theorem~\ref{theorem2}}
\label{subsec:thm4_appendix}%
According to Lemma \ref{lemma_g_star}, when a cross-entropy loss or mean squared error is used for $\ell$, the mapping $\boldsymbol{g}^{\star}(\boldsymbol{x})=\bar{\boldsymbol{\eta}}(\boldsymbol{x})$ is the unique minimizer of $R_{\rm{surr}}(\boldsymbol{g};\ell)$.
Let $\boldsymbol{g(x)}=\boldsymbol{T}(f(\boldsymbol{x}))$, since $\boldsymbol{g}^{\star}\in\cG$, $R_{\rm{surr}}(\boldsymbol{g}(\boldsymbol{x}))=\mathbb{E}_{(\boldsymbol{x},\Bar{y})\sim\Bar{\mathcal{D}}}[\ell(\boldsymbol{g}(\boldsymbol{x}), \bar{y})]$ achieves its minimum if and only if $\boldsymbol{g}(\boldsymbol{x})=\boldsymbol{\bar{\eta}}(\boldsymbol{x})=\boldsymbol{g^{\star}}(\boldsymbol{x})$. 
Combining this result with Theorem \ref{theorem1} and Lemma \ref{lemma_injective}, we then obtain that $\boldsymbol{g}(\boldsymbol{x})=\bar{\boldsymbol{\eta}}(\boldsymbol{x})$ if and only if $f(\boldsymbol{x})=\eta(\boldsymbol{x})$.
Since
\begin{align*}
R_{\rm{surr}}(f)&=\mathbb{E}_{(\boldsymbol{x},\Bar{y})\sim\Bar{\mathcal{D}}}[\ell(\boldsymbol{T}(f(\boldsymbol{x})), \bar{y})]\notag\\
&=\mathbb{E}_{(\boldsymbol{x},\Bar{y})\sim\Bar{\mathcal{D}}}[{\ell}(\boldsymbol{g}(\boldsymbol{x}), \bar{y})]= R_{\rm{surr}}(\boldsymbol{g}),
\end{align*}
$f_{\rm{surr}}^\star$ is induced by $\boldsymbol{g}^{\star}=\argmin_{\boldsymbol{g}}R_{\rm{surr}}(\boldsymbol{g})$.
So we have $f_{\rm{surr}}^\star(\boldsymbol{x}) = \argmin_{f}R_{\rm{surr}}(f) = \eta(\boldsymbol{x})$.
    
On the other hand, when $\ell_{\rm{b}}$ is a cross-entropy loss, i.e., the logistic loss in the binary case, or mean squared error, the mapping $f^{\star}$ is the unique minimizer of $R(f;\ell_{\rm{b}})$. 
We skip the proof since it is similar to the proof of Lemma \ref{lemma_g_star}. 
So we obtain that $f^{\star}(\boldsymbol{x})=\eta(\boldsymbol{x})=f_{\rm{surr}}^{\star}(\boldsymbol{x})$, which concludes the proof.
\qed

\subsection{Proof of Lemma~\ref{lemma4}}
\label{subsec:lemma4_proof}%
$\forall j \in {1,\ldots,m}$, by taking derivative of $T_j(t)$ with respect to $t$, we obtain
    \begin{equation}
            \left|\frac{\partial T_j}{\partial t}\right|=\frac{|a_j d-b_j c|}{(c\cdot t+d)^2},
        \label{eq:derivative}
    \end{equation}
    where
    \begin{equation} \nonumber
        \begin{gathered}
            a_j = {\rho}_j(\pi_j-\pi_{\mathcal{D}}),\quad b_j = {\rho}_j\pi_{\mathcal{D}}(1-\pi_j),\quad c = \sum_{j=1}^m{\rho}_j(\pi_j-\pi_{\mathcal{D}}),\quad d = \sum_{j=1}^m{\rho}_j\pi_{\mathcal{D}}(1-\pi_j).
        \end{gathered}
    \end{equation}
    Since for all class priors we have $0\leq \pi_j\leq 1$, $0<\pi_D<1$, $0< \rho_j< 1$, $\sum_{j=1}^m\rho_j=1$, and $\exists j, j' \in \{1,\ldots,m\}$ such that $j \neq j' $ and $\pi_j \neq \pi_{j'}$, obviously we can obtain
    \begin{equation} \nonumber
            -1 \leq a_j \leq 1,\quad 0 \leq b_j \leq 1,\quad -1 \leq c \leq 1,\quad {\rm{and}}\ 0 < d \leq 1.
    \end{equation}
    Therefore, the numerator of \eqref{eq:derivative} satisfies
    \begin{equation}
        |a_jd-b_j c| \leq 2.
    \label{numerator}
    \end{equation}
    On the other hand, since $d > 0$ and $0 \leq t\leq 1$ , by substituting $t=0$ and $t=1$ respectively, we can obtain
    \begin{equation}\nonumber
    \begin{gathered}
        c\cdot t + d \geq c + d=\sum_{j=1}^m{\rho}_j\pi_j(1-\pi_{\mathcal{D}}) > 0,\quad {\rm{if}}\  c < 0; \\
        c\cdot t + d \geq d > 0,\quad {\rm{if}}\  c \geq 0. \\
    \end{gathered}
    \end{equation}
    Next we lower bound this term by $c\cdot t + d \geq \min(c+d,d) > 0$. 
    As a result, the denominator of \eqref{eq:derivative} satisfies
    \begin{equation}
        \begin{split}
            (c\cdot t+d)^2 &\geq \left(\min(c+d,d)\right)^2 \\
            &= \left(\min\left(\sum_{j=1}^m{\rho}_j\pi_j(1-\pi_{\mathcal{D}}),\sum_{j=1}^m{\rho}_j\pi_{\mathcal{D}}(1-\pi_j)\right)\right)^2 \\
            &= \alpha^2.
        \end{split}
        \label{denominator}
    \end{equation}
    Then, by combing \eqref{numerator} and \eqref{denominator}, we have
    \begin{equation} \nonumber
        \left|\frac{\partial T_j}{\partial t}\right| \leq \frac{2}{\alpha^2}.
    \end{equation}
    This bound illustrates that $T_j(f(\boldsymbol{x}))$ is Lipschitz-continuous with respect to $f(\boldsymbol{x})$ with a Lipschitz constant $2/\alpha^2$ and we complete the proof.
    \qed

\subsection{Proof of Theorem~\ref{error_bound}}
\label{subsec:thm4_proof}%
We first introduce the following lemmas which are useful to derive the estimation error bound.
\begin{lemma} [Uniform deviation bound]
    Let $\boldsymbol{g} \in \mathcal{G}$, where $\mathcal{G}=\{\boldsymbol{x}\mapsto \boldsymbol{T}(f(\boldsymbol{x}))\mid f \in \mathcal{F} \}$ is a class of measurable functions,
    ${\mathcal{X}}_{\rm{tr}} = \{(\boldsymbol{x}_i, \Bar{y}_i)\}_{i=1}^{n_{\rm{tr}}}\stackrel{\mathrm{i.i.d.}}{\sim}\Bar{\mathcal{D}}$ be a fixed sample of size $n_{\rm{tr}}$ i.i.d. drawn from $\bar{\mathcal{D}}$,
    and $\{\sigma_1,\ldots,\sigma_{n_{\rm{tr}}}\}$ be the Rademacher variables, i.e., independent uniform random variables taking values in $\{-1, 1\}$.
    Let $\mathfrak{R}_{n_{\rm{tr}}}({\ell}\circ\mathcal{G})$ be the Rademacher complexity of ${\ell}\circ\mathcal{G}$ which is defined as
    \begin{equation} \nonumber
        \mathfrak{R}_{n_{\rm{tr}}}({\ell}\circ\mathcal{G})=\mathbb{E}\left[\sup_{\boldsymbol{g}\in\mathcal{G}}\frac{1}{n_{\rm{tr}}}\sum_{i=1}^{n_{\rm{tr}}}\sigma_i{\ell}(\boldsymbol{g}(\boldsymbol{x}_i),\bar{y}_i)\right].
    \end{equation}
    Under the assumptions of Theorem~\ref{error_bound}, ${\ell}(\boldsymbol{g}(\boldsymbol{x}), \bar{y})$ is upper-bounded by $M_{\ell}$. 
    Then, for any $\delta>0$, we have with probability at least $1 - \delta$,
    \begin{equation}\nonumber
        \sup_{\boldsymbol{g}\in \mathcal{G}}|\widehat{R}_{\rm{surr}}(\boldsymbol{g})-R_{\rm{surr}}(\boldsymbol{g})| \leq 2\mathfrak{R}_{n_{\rm{tr}}}({\ell}\circ\mathcal{G})+M_{\ell}\sqrt{\frac{\ln(2/\delta)}{2n_{\rm{tr}}}}.
    \end{equation}
    \label{lemma7}
    \end{lemma}
    \begin{proof}
        We consider the one-side uniform deviation $\sup_{\boldsymbol{g}\in\mathcal{G}}\widehat{R}_{\rm{surr}}(\boldsymbol{g})-R_{\rm{surr}}(\boldsymbol{g})$. Suppose that a sample $(\boldsymbol{x}_i, \bar{y}_i)$ is replaced by another arbitrary sample $(\boldsymbol{x}_j, \bar{y}_j)$, the change of $\sup_{\boldsymbol{g}\in\mathcal{G}}\widehat{R}_{\rm{surr}}(\boldsymbol{g})-R_{\rm{surr}}(\boldsymbol{g})$ is no more than $M_{\ell}/n_{\rm{tr}}$, since the loss ${\ell}(\cdot)$ is bounded by $M_{\ell}$. 
        By applying the McDiarmid's inequality \cite{mcdiarmid1989method}, for all $\epsilon'>0$ we have
        \begin{align*}
            \pr\{\sup\nolimits_{\boldsymbol{g}\in\mathcal{G}}\widehat{R}_{\rm{surr}}(\boldsymbol{g})-R_{\rm{surr}}(\boldsymbol{g})-\mathbb{E}[\sup\nolimits_{\boldsymbol{g}\in\mathcal{G}}\widehat{R}_{\rm{surr}}(\boldsymbol{g})-R_{\rm{surr}}(\boldsymbol{g})]\geq\epsilon'\}\leq\exp{\left(\frac{-2n_{\rm{tr}}\epsilon'^{2}}{M_{\ell}^2}\right)}.
        \end{align*}
        Equivalently, for any $\delta>0$, with probability at least $1-\delta/2$,
        \begin{equation} \nonumber
            \sup_{\boldsymbol{g}\in\mathcal{G}}\widehat{R}_{\rm{surr}}(\boldsymbol{g})-R_{\rm{surr}}(\boldsymbol{g})\leq\mathbb{E}\left[\sup_{\boldsymbol{g}\in\mathcal{G}}\widehat{R}_{\rm{surr}}(\boldsymbol{g})-R_{\rm{surr}}(\boldsymbol{g})\right]+M_{\ell}\sqrt{\frac{\ln(2/\delta)}{2n_{\rm{tr}}}}.
        \end{equation}
        By \emph{symmetrization} \cite{vapnik98SLT}, it is a routine work to show that
        \begin{equation} \nonumber
        \mathbb{E}\left[\sup_{\boldsymbol{g}\in\mathcal{G}}\widehat{R}_{\rm{surr}}(\boldsymbol{g})-R_{\rm{surr}}(\boldsymbol{g})\right]\leq 2\mathfrak{R}_{n_{\rm{tr}}}({\ell}\circ\mathcal{G}).
        \end{equation}
        The other side uniform deviation $\sup_{\boldsymbol{g}\in\mathcal{G}}R_{\rm{surr}}(\boldsymbol{g})-\widehat{R}_{\rm{surr}}(\boldsymbol{g})$ can be bounded similarly.
        By combining the two sides' inequalities, we complete the proof.
    \end{proof}
    
\begin{lemma}
    Let $f \in \mathcal{F}$, where $\mathcal{F}=\{f:\mathcal{X} \rightarrow\mathbb{R}\}$ is a class of measurable functions,
    $\{\boldsymbol{x}_i\}_{i=1}^{n_{\rm{tr}}}\stackrel{\mathrm{i.i.d.}}{\sim}p_{\rm{tr}}(\boldsymbol{x})$ be a fixed sample of size $n_{\rm{tr}}$ i.i.d. drawn from the marginal density $p_{\rm{tr}}(\boldsymbol{x})$,
    and $\{\sigma_1,\ldots,\sigma_{n_{\rm{tr}}}\}$ be the Rademacher variables.
    Let $\mathfrak{R}_{n_{\rm{tr}}}(\mathcal{F})$ be the Rademacher complexity of $\mathcal{F}$ which is defined as
    \begin{equation} \nonumber
        \mathfrak{R}_{n_{\rm{tr}}}(\mathcal{F})=\mathbb{E}\left[\sup_{f\in\mathcal{F}}\frac{1}{n_{\rm{tr}}}\sum_{i=1}^{n_{\rm{tr}}}\sigma_i f(\boldsymbol{x}_i) \right].
    \end{equation}
    Then we have
    \begin{align*}
        \mathfrak{R}_{n_{\rm{tr}}}({\ell}\circ\mathcal{G})\leq \frac{2\sqrt{2}m\mathcal{L}_{\ell}}{\alpha^2}\mathfrak{R}_{n_{\rm{tr}}}(\mathcal{F}).
    \end{align*}
    \label{lemma8}
    \end{lemma}
    \begin{proof}
   In what follows, we upper-bound $\mathfrak{R}_{n_{\rm{tr}}}({\ell}\circ\mathcal{G})$. 
   Since $\ell(\boldsymbol{g}(\boldsymbol{x}), \bar{y})$ is $\mathcal{L}_{\ell}$-Lipschitz continuous w.r.t $\boldsymbol{g}$,
   according to the Rademacher vector contraction inequality \cite{maurer2016vector}, we have
   \begin{align}
        \mathfrak{R}_{n_{\rm{tr}}}({\ell}\circ\mathcal{G})&=\mathbb{E}\left[\sup_{\boldsymbol{g}\in\mathcal{G}}\frac{1}{n_{\rm{tr}}}\sum_{i=1}^{n_{\rm{tr}}}\sigma_i{\ell}(\boldsymbol{g}(\boldsymbol{x}_i),\bar{y}_i)\right]\notag\\
        &\leq \frac{\sqrt{2}\mathcal{L}_{\ell}}{n_{\rm{tr}}}\cdot\mathbb{E}\left[\sup_{\boldsymbol{g}\in\mathcal{G}}\sum_{i=1}^{n_{\rm{tr}}}\sum_{j=1}^{m}\sigma_{ij}g_j(\boldsymbol{x}_i)\right]\notag\\
        \label{rademacherupper}
        &\leq \frac{\sqrt{2}\mathcal{L}_{\ell}}{n_{\rm{tr}}}\cdot\sum_{j=1}^{m} \mathbb{E}\left[\sup_{\boldsymbol{g}\in\mathcal{G}}\sum_{i=1}^{n_{\rm{tr}}}\sigma_{ij}g_j(\boldsymbol{x}_i)\right],
    \end{align}
    where $g_j(\boldsymbol{x_i})$ is the $j$-th component of $\boldsymbol{g}(\boldsymbol{x_i})$, and $\sigma_{ij}$ are an $n_{\rm{tr}}\times m$ matrix of independent Rademacher variables.
    As shown in Lemma~\ref{lemma4}, $g_j(\boldsymbol{x}) = T_j(f(\boldsymbol{x}))$ and $T_j(f)$ is Lipschitz continuous w.r.t $f$ with a Lipschitz constant $2/\alpha^2$.
    Then we apply the Talagrand’s contraction lemma \cite{sshwartz14UML} and obtain
    \begin{equation}\nonumber
    \begin{split}
    \sum_{j=1}^{m} \mathbb{E}\left[\sup_{\boldsymbol{g}\in\mathcal{G}}\sum_{i=1}^{n_{\rm{tr}}}\sigma_{ij}g_j(\boldsymbol{x}_i)\right]
        &= \sum_{j=1}^m\mathbb{E}\left[\sup_{f \in \mathcal{F}}\sum_{i=1}^{n_{\rm{tr}}}\sigma_{ij}T_j(f(\boldsymbol{x_i}))\right]\\
        &\leq \frac{2}{\alpha^2}\sum_{j=1}^m\mathbb{E}\left[\sup_{f \in \mathcal{F}}\sum_{i=1}^{n_{\rm{tr}}}\sigma_{ij}f(\boldsymbol{x_i})\right]\\
        &=\frac{2mn_{\rm{tr}}}{\alpha^2}\mathfrak{R}_{n_{\rm{tr}}}(\mathcal{F}).
    \end{split}
    \label{rademacher_vector}
    \end{equation}
    By substituting it into \eqref{rademacherupper}, we complete the proof.
\end{proof}

    Based on Lemma \ref{lemma7} and Lemma \ref{lemma8}, the estimation error bound is proven through
    \begin{equation} \nonumber
        \begin{split}
    &R_{\rm{surr}}(\hat{f}_{\rm{surr}})-R_{\rm{surr}}(f^{\star}_{\rm{surr}})\\
    &~~~=\left(\widehat{R}_{\rm{surr}}(\hat{f}_{\rm{surr}}) - \widehat{R}_{\rm{surr}}(f^{\star}_{\rm{surr}})\right)+ \left(R_{\rm{surr}}(\hat{f}_{\rm{surr}}) - \widehat{R}_{\rm{surr}}(\hat{f}_{\rm{surr}})\right)+ \left(\widehat{R}_{\rm{surr}}(f^{\star}_{\rm{surr}})  - R_{\rm{surr}}(f^{\star}_{\rm{surr}})\right)\\
    &~~~\leq \left(R_{\rm{surr}}(\hat{f}_{\rm{surr}}) - \widehat{R}_{\rm{surr}}(\hat{f}_{\rm{surr}})\right) + \left(\widehat{R}_{\rm{surr}}(f^{\star}_{\rm{surr}})  - R_{\rm{surr}}(f^{\star}_{\rm{surr}})\right)\\
    &~~~\leq 2\sup_{f\in \mathcal{F}}|\widehat{R}_{\rm{surr}}(f)-R_{\rm{surr}}(f)|\\
            &~~~= 2\sup_{\boldsymbol{g}\in \mathcal{G}}|\widehat{R}_{\rm{surr}}(\boldsymbol{g})-R_{\rm{surr}}(\boldsymbol{g})|\\
            &~~~\leq 4\mathfrak{R}_{n_{\rm{tr}}}({\ell}\circ\mathcal{G})+2M_{\ell}\sqrt{\frac{\ln(2/\delta)}{2n_{\rm{tr}}}}\\
            &~~~\leq \frac{8\sqrt{2}m\mathcal{L}_{\ell}}{\alpha^2}\mathfrak{R}_{n_{\rm{tr}}}(\mathcal{F})+2M_{\ell}\sqrt{\frac{\ln(2/\delta)}{2n_{\rm{tr}}}},
        \end{split}
    \end{equation}
    where the second equality is due to that $\mathcal{G}=\{\boldsymbol{x}\mapsto \boldsymbol{T}(f(\boldsymbol{x}))\mid f \in \mathcal{F} \}$ and $\boldsymbol{T}(\cdot)$ is deterministic. 
    \qed
\section{Supplementary Information on the Experiments}
\label{setup}%
In this appendix, we provide supplementary information on the experiments.

\subsection{Datasets}
\label{setup-datasets}%
We describe details of the datasets as follows.

    \paragraph{MNIST} 
    This is a dataset of normalized grayscale images containing handwritten digits from 0 to 9. 
    All the images are fitted into a 28 $\times$ 28 pixels. 
    The total number of training images and test images is 60,000 and 10,000 respectively.
    We use the even digits as the positive class and odd digits as the negative class.
    
    \paragraph{Fashion-MNIST}
    This is a dataset of grayscale images of different types of modern clothes. 
    All the images are of the size 28 $\times$ 28 pixels. 
    Similar to MNIST, this dataset has 60,000 training images and 10,000 test images.
    We convert this 10-class dataset into a binary dataset as follows:
    \begin{itemize}
        \item The classes ‘Pullover’, ‘Dress’, ‘T-shirt’, ‘Trouser’, ‘Shirt’, ‘Bag’, ‘Ankle boot’ and ‘Sneaker’ are denoted as the positive class;
        \item The classes  ‘Coat’ and ‘Sandal’ are denoted as the negative class.
    \end{itemize}
    
    \paragraph{Kuzushiji-MNIST}
    This is a dataset of grayscale images of cursive Japanese (Kuzushiji) characters. 
    This dataset also has all images of size 28 $\times$ 28. 
    And the total number of training images and test images is 60,000 and 10,000 respectively.
    We convert this 10-class dataset into a binary dataset as follows:
    \begin{itemize}
        \item The classes ‘ki’, ‘re’, and ‘wo’ are denoted as the positive class;
        \item The classes ‘o’, ‘su’, ‘tsu’, ‘na’, ‘ha’, ‘ma’, and ‘ya’ are denoted as the negative class.
    \end{itemize}
    
    \paragraph{CIFAR-10}
    This dataset is made up of color images of ten types of objects and animals. 
    The size of all images in this dataset is 32 $\times$ 32.
    There are 5,000 training images and 1,000 test images for each class, so 50,000 training and 10,000 test images in total.
    We convert this 10-class dataset into a binary dataset as follows:
    \begin{itemize}
        \item The positive class consists of ‘airplane’, ‘bird’, ‘deer’, ‘dog’, ‘frog’, ‘cat’, and ‘horse’;
        \item The negative class consists of ‘automobile’, ‘ship’, and ‘truck’.
    \end{itemize}
    
The generation of each U set is the same for all four benchmark datasets. 
More specifically, given the number of U sets $m$, class priors $\{\pi_j\}_{j=1}^{m}$, and the set sizes $\{n_j\}_{j=1}^{m}$, for $j$-th U set, we go through the following process:
    \begin{enumerate}
        \item Randomly shuffle the benchmark dataset;
        \item Randomly select $n_j^p = n_j \times \pi_j$ samples of positive class;
        \item Randomly select $n_j^n= n_j - n_j^p$ samples of negative class;
        \item Combine them and we obtain the $j$-th U set.
    \end{enumerate}
    
\subsection{Models}
\label{setup-models}%
We describe details of the model architecture and optimization algorithm as follows.

    \paragraph{MLP} It is a 5-layer fully connected perceptron with ReLU \cite{nair2010rectified} as the activation function. The model architecture was $d-300-300-300-1$, where $d$ is the dimension of the input.
    Batch normalization \cite{ioffe2015batch} was applied before each hidden layer and $\ell_2$-regularization was added. Dropout \cite{srivastava2014dropout} with rate 0.2 was also added before each hidden layer. The optimizer was Adam \cite{kingma2014adam} with the default momentum parameters ($\beta_1 = 0.9$ and $\beta_2 = 0.999$).
    
    \paragraph{ResNet-32} It is a 32-layer residual network \cite{he16cvpr} and the architecture was as follows:\\
        0th (input) layer: $(32*32*3)-$ \\
        1st to 11th layers: $C(3*3, 16)-[C(3*3, 16), C(3*3, 16)]*5-$  \\
        12th to 21st layers: $[C(3*3, 32), C(3*3,32)]*5-$\\
        22nd to 31st layers: $[C(3*3, 64), C(3*3, 64)]*5-$\\
        32nd layer: Global Average Pooling$-1$,\\
    where $C(3*3, 96)$ represents a 96-channel of $3*3$ convolutions followed by a ReLU activation function, $[\cdot]$*2 represents a repeat of twice of such layer, $C(3*3, 96, 2)$ represents a similar layer but with stride 2, and $[\cdot, \cdot]$ represents a building block. Batch normalization was applied for each hidden layers and $\ell_2$-regularization was also added. The optimizer was Adam with the default momentum parameters ($\beta_1 = 0.9$ and $\beta_2 = 0.999$).
    
    The MLP model was used for the MNIST, Fashion-MNIST, Kuzushiji-MNIST dataset, and the ResNet-32 model was used for the CIFAR-10 dataset.

\subsection{Other Details}
\label{setup-otherdetail}%
 We implemented all the methods by Keras and conducted all the experiments on an NVIDIA Tesla P100 GPU. 
 The batch size was 256 for all the methods.
 For MNIST, Fashion-MNIST, and Kuzushiji MNIST dataset, the initial learning-rate was 1e-5 for U$^m$-SSC and 1e-4 for the MMC based methods and LLP-VAT. 
 For CIFAR-10 dataset, the initial learning-rate was 5e-6 for U$^m$-SSC and 1e-5 for the MMC based methods and LLP-VAT. 
 In addition, the learning rate was decreased by $1/(1+\textrm{decay}\cdot\textrm{epoch})$, where the decay parameter was 1e-4. 
 This is the built-in learning rate scheduler of Keras.
 
 We describe details of the hyper-parameters for the baseline methods as follows.
 \begin{itemize}
 \item MMC$\text{-}{\rm{U^2}\text{-}\rm{b}}$ \cite{scott2020learning}: by assuming that the number of sets $m=2k$, this baseline method firstly pairs all the U sets and then linearly combines the unbiased balanced risk estimator of each pair, The learning objective is 
 \begin{align*}
     \widehat{R}_{\text{MMC-}{\rm{U^2}\text{-}\rm{b}}}(f)=\sum\nolimits_{j=1}^{k}\omega_j\widehat{R}_{\rm{U^2}\text{-}\rm{b}}(f),
 \end{align*}
  where
\begin{align*}
&\widehat{R}_{\rm{U^2}\text{-}\rm{b}}(f)=
\frac{c_{b1}^+}{n}\sum_{i=1}^{n_1}
\ell_b(f(\boldsymbol{x}_{i}^1),+1)
-\frac{c_{b2}^+}{n}\sum_{j=1}^{n_2}
\ell_b(f(\boldsymbol{x}_{j}^2),+1)\notag\\
&~~~~~~~~~~~~~~~~
-\frac{c_{b1}^-}{n}\sum_{i=1}^{n_1}
\ell_b(f(\boldsymbol{x}_{i}^1),-1)
+\frac{c_{b2}^-}{n}\sum_{j=1}^{n_2}
\ell_b(f(\boldsymbol{x}_{j}^2),-1),
\end{align*}
$c_{b1}^+=\frac{1-\pi_2}{2(\pi_1-\pi_2)}$, $c_{b1}^-=\frac{\pi_2}{2(\pi_1-\pi_2)}$, $c_{b2}^+=\frac{1-\pi_1}{2(\pi_1-\pi_2)}$, and $c_{b2}^-=\frac{\pi_1}{2(\pi_1-\pi_2)}$.
 For the pairing process, since we use the uniform set sizes, i.e., the set size of each U set is the same as $n_{\rm{tr}}/m$ ($n_{\rm{tr}}=60,000$ in MNIST, Fashion-MNIST, and Kuzushiji-MNIST, $n_{\rm{tr}}=50,000$ in CIFAR-10), we pair all the U sets following Proposition 9 in Appendix S6 of \citet{scott2020learning}, i.e., match the U set with the largest class prior $\pi_j$ with the smallest, the U set with the second largest class prior $\pi_j$ with the second smallest, and so on.
 For the combination weights, we set them following Theorem 5 in Section 2.2 of \citet{scott2020learning}.
 More specifically, for the $j$-th pair of U sets: $\mathcal{X}_{\rm{tr}}^{1}$ and $\mathcal{X}_{\rm{tr}}^{2}$, assume $\pi_1>\pi_2$, since we use uniform set sizes, the optimal weights $\omega_j\propto(\pi_1-\pi_2)^2$.
 So we set the weight $\omega_j$ as $(\pi_1-\pi_2)^2$ and then normalize all of them to sum to 1, i.e., $\sum_{j=1}^k\omega_j=1$.
 
\item MMC$\text{-}{\rm{U^2}}$: this method improves the MMC$\text{-}{\rm{U^2}\text{-}\rm{b}}$ baseline by replacing the unbiased balanced risk estimator $\widehat{R}_{\rm{U^2}\text{-}\rm{b}}(f)$ with the unbiased risk estimators $\widehat{R}_{\rm{U^2}}(f)$ \cite{lu2018minimal}. The learning objective is
 \begin{align*}
     \widehat{R}_{\text{MMC-}{\rm{U^2}}}(f)&=\sum\nolimits_{j=1}^{k}\omega_j\widehat{R}_{\rm{U^2}}(f),
 \end{align*}
 where
 \begin{align*}
     &\widehat{R}_{\rm{U^2}}(f)=
\underbrace{
\frac{c_1^+}{n}\sum_{i=1}^{n_1}
\ell_b(f(\boldsymbol{x}_{i}^1),+1)
-\frac{c_2^+}{n}\sum_{j=1}^{n_2}
\ell_b(f(\boldsymbol{x}_{j}^2),+1)}_{\widehat{R}_{\rm{U^2}\text{-}\rm{p}}(f)}\notag\\
&~~~~~~~~~~~~~~~~\underbrace{
-\frac{c_1^-}{n}\sum_{i=1}^{n_1}
\ell_b(f(\boldsymbol{x}_{i}^1),-1)
+\frac{c_2^-}{n}\sum_{j=1}^{n_2}
\ell_b(f(\boldsymbol{x}_{j}^2),-1)}_{\widehat{R}_{\rm{U^2}\text{-}\rm{n}}(f)},
\end{align*}
$c_1^+=\frac{(1-\pi_2)\pi_{\mathcal{D}}}{\pi_1-\pi_2}$, $c_1^-=\frac{\pi_2(1-\pi_{\mathcal{D}})}{\pi_1-\pi_2}$, $c_2^+=\frac{(1-\pi_1)\pi_{\mathcal{D}}}{\pi_1-\pi_2}$, and $c_2^-=\frac{\pi_1(1-\pi_{\mathcal{D}})}{\pi_1-\pi_2}$.
The pairing process and the combination weights setup follow those of MMC$\text{-}{\rm{U^2}\text{-}\rm{b}}$.
 
\item MMC$\text{-}{\rm{U^2}\text{-}\rm{c}}$: this method improves the MMC$\text{-}{\rm{U^2}}$ baseline by replacing the unbiased risk estimators $\widehat{R}_{\rm{U^2}}(f)$ with the non-negative risk estimators $\widehat{R}_{\rm{U^2}\text{-}\rm{c}}(f)$ \cite{lu2020mitigating}. The learning objective is
 \begin{align*}
     \widehat{R}_{\text{MMC-}{\rm{U^2}\text{-}\rm{c}}}(f)=\sum\nolimits_{j=1}^{k}\omega_j\widehat{R}_{\rm{U^2}\text{-}\rm{c}}(f),
 \end{align*}
 where
 \begin{align*}
\widehat{R}_{\rm{U^2}\text{-}\rm{c}}(f)= f_\mathrm{c}(\widehat{R}_{\rm{U^2}\text{-}\rm{p}}(f))+f_\mathrm{c}(\widehat{R}_{\rm{U^2}\text{-}\rm{n}}(f)).
\end{align*}
According to \citet{lu2020mitigating}, the \emph{generalized leaky ReLU} function, i.e., \[f_\mathrm{c}(r)=
  \begin{cases}
    r&(r\geq 0),\\
    -\kappa r&(r< 0),\\
  \end{cases}
  \]
  for $\kappa\geq 0$, works well as the correction function $f_\mathrm{c}$, so we choose it for implementing this baseline method.
  The hyper-parameter $\kappa$ was chosen based on a validation dataset, and the pairing process and the combination weights setup follow those of MMC$\text{-}{\rm{U^2}\text{-}\rm{b}}$.
  
\item LLP-VAT \cite{tsai2020learning}: this baseline method is based on empirical proportion risk minimization. The learning objective is
\begin{align*}
\widehat{R}_{\rm{prop}\text{-}\rm{c}}(f)=\widehat{R}_{\rm{prop}}(f)+\alpha\ell_{\rm{cons}}(f),
\end{align*}
where 
\begin{align*}
\widehat{R}_{\rm{prop}}(f)=\sum\nolimits_{j=1}^{m}d_{\rm{prop}}(\pi_j, \hat{\pi}_j)
\end{align*}
is the proportion risk,
$\pi_j$ and
\begin{align*}
  \hat{\pi}_j=\frac{1}{n_j}\sum_{i=1}^{n_j}\frac{1+\sign(f(\boldsymbol{x}_i^j)-1/2)}{2}  
\end{align*}
are the true and predicted label proportions for the $j$-th U set $\mathcal{X}_{\rm{tr}}^{j}$, 
$d_{\rm{prop}}$ is a distance function, and
\begin{align*}
    \ell_{\rm{cons}}(f)=d_{\rm{cons}}(f(\boldsymbol{x}), f(\boldsymbol{\hat{x}}))
\end{align*}
is the consistency loss, $d_{\rm{cons}}$ is a distance function, $\boldsymbol{\hat{x}}$ is a perturbed input from the original one $\boldsymbol{x}$.
We set the hyper-parameters $\alpha=0.05$ and the perturbation weight $\mu=6.0$ for LLP-VAT following the default implementation in their paper \cite{tsai2020learning}.
\end{itemize}
 

\section{Supplementary Experimental Results}
\label{suppresults}%
In this appendix, we provide supplementary experimental results.

\subsection{Comparison with State-of-the-art Methods}
\label{negativeloss}%
Please find Table \ref{test_10_50_table} the final classification errors of comparing our proposed method with state-of-the-art methods on learning from 10, 25, and 50 U sets (corresponds to Figure~\ref{fig:performance}).

\begin{table*}[t]\centering
    \caption{Means (standard deviations) of the classification error over three trials in percentage of each method on learning from 10, 25 and 50 U sets. Best and comparable methods (paired \textit{t}-test at significance level 5\%) are highlighted in boldface.}
    \vspace{1ex}
    \newcommand{\tabincell}[2]{\begin{tabular}{@{}#1@{}}#2\end{tabular}}
        \begin{tabular}{ c|c | c c c c c} 
        \hline
        Dataset & \tabincell{c}{Sets} & \tabincell{c}{MMC$\text{-}{\rm{U^2}\text{-}\rm{b}}$} & \tabincell{c}{MMC$\text{-}{\rm{U^2}}$} & \tabincell{c}{MMC$\text{-}{\rm{U^2}\text{-}\rm{c}}$} & \tabincell{c}{LLP-VAT} & \tabincell{c}{U$^m$-SSC} \\ 
        \hline
        \multirowcell{3}{MNIST}      
        & 10 & 7.7(0.55) & 8.03(0.74) & 4.46(0.23) & 3.62(0.38) & \textbf{3.05(0.08)}\\
        & 25 & 5.35(0.22) & 5.32(0.28) & 3.69(0.11) & 3.28(0.35) & \textbf{2.51(0.02)}\\
        & 50 & 5.81(0.22) & 5.82(0.12) & 3.29(0.09) & 3.02(0.22) & \textbf{2.86(0.04)}\\
		\hline
		\multirowcell{3}{Fashion-\\MNIST}      
		& 10 & 16.63(1.38) & 9.49(0.37) & 8.12(0.51) & 21.23(3.52) & \textbf{6.5(0.21)}\\
        & 25 & 11.1(0.45) & 9.12(0.1) & 7.45(0.1) & 26.66(0.4) & \textbf{6.14(0.02)}\\
        & 50 & 11.18(0.53) & 9.6(0.47) & 8.52(0.48) & 27.92(2.22) & \textbf{6.6(0.06)}\\
		\hline
		\multirowcell{3}{Kuzushiji-\\MNIST}      
        & 10 & 16.25(0.61) & 15.23(0.3) & 12.88(0.35) & 16.12(0.41) & \textbf{9.83(0.4)}\\
        & 25 & 15.93(0.71) & 14.02(0.12) & 10.18(0.33) & 19.48(1.84) & \textbf{8.98(0.07)}\\
        & 50 & 15.8(0.37) & 12.46(0.43) & 9.69(0.37) & 18.94(0.4) & \textbf{8.97(0.52)}\\
		\hline
		\multirowcell{3}{CIFAR-10}      
        & 10 & 15.83(0.21) & 16.01(0.32) & 14.33(0.06) & 19.38(0.05) & \textbf{13.43(0.14)}\\
        & 25 & 19.6(0.77) & 16.18(0.27) & 14.19(0.25) & 16.89(0.15) & \textbf{13.31(0.13)}\\
        & 50 & 21.1(1.03) & 16.08(0.38) & 14.28(0.13) & 17.66(0.57) & \textbf{13.32(0.19)}\\
        \hline
        \end{tabular}
        \label{test_10_50_table}
    \end{table*}

In the experiments, we also find that the empirical training risk of the proposed U$^m$-SSC is obviously higher than all other baseline methods. This is due to the added transition layer and the rescales the output range. We provide a detailed explanation as follows.

By using the monotonicity of the transition function $T_j(\cdot)$ \cite{menon15icml}, we can compute the range of the model output. 
Since $g(\boldsymbol{x})\in [0, 1]$, by plugging in $g(\boldsymbol{x})=0$ and $g(\boldsymbol{x})=1$ respectively we obtain
    \begin{equation}\nonumber
        \begin{split}
            &T_j(0) = \frac{b_j}{d}=\frac{\rho_j\pi_{\mathcal{D}}(1-\pi_j)}{\sum_{j=1}^m\rho_j\pi_{\mathcal{D}}(1-\pi_j)},\\
            &T_j(1) = \frac{a_j + b_j}{c_j+d}=\frac{\rho_j\pi_{j}(1-\pi_{\mathcal{D}})}{\sum_{j=1}^m\rho_j\pi_{j}(1-\pi_{\mathcal{D}})}. 
        \end{split}
    \end{equation} 
According to our generation processes of class priors and set size, $\pi_j \in [0.1, 0.9]$ and $\rho_j = 1/m$ for any $j=1,\ldots,m$. 
The upper bound of the model output $\max(T_j(0),T_j(1))$ takes value between 0.01 and 0.1. 
As a result, the cross-entropy loss gives its value in range $[2.3, 4.6]$, which is relatively high than usual training loss.
We note that this high training loss has an effect on hyper-parameters tuning, especially for the learning rate. 
We may need a relatively small learning rate for better performance of our method.
    
\subsection{Robustness against Inaccurate Class Priors}
\label{inaccurate}
Please find Table \ref{robust_table} the final classification errors of our method on learning from 50 U sets with inaccurate class priors (corresponds to Figure~\ref{fig:robust}).

    \begin{table}[t]\centering
    \caption{Means (standard deviations) of the classification error over three trials in percentage for the U$^m$-SSC method tested on inaccurate class priors.}
        \newcommand{\tabincell}[2]{\begin{tabular}{@{}#1@{}}#2\end{tabular}}
        \begin{tabular}{ c|c | c c c c c} 
        \hline
        Dataset & \tabincell{c}{Sets} & \tabincell{c}{True} & \tabincell{c}{$\epsilon=0.05$} & \tabincell{c}{$\epsilon=0.1$} & \tabincell{c}{$\epsilon=0.15$} & \tabincell{c}{$\epsilon=0.2$}\\ 
        \hline
        \multirowcell{2}{MNIST}      
        & 10 & 2.54(0.02) & 2.64(0.06) & 3.31(0.18) & 2.98(0.14) & 3.84(0.25) \\
		& 50 & 2.45(0.04) & 2.52(0.02) & 2.69(0.04) & 3.11(0.19) & 3.16(0.13)\\
		\hline
		\multirowcell{2}{Fashion-\\MNIST}      
		& 10 & 6.22(0.05) & 6.31(0.03) & 6.13(0.13) & 6.61(0.04) & 9.39(0.19) \\
		& 50 & 6.37(0.26) & 6.39(0.17) & 6.76(0.11) & 7.64(0.22) & 10.91(0.47)\\
		\hline
		\multirowcell{2}{Kuzushiji-\\MNIST}      
        & 10 & 8.74(0.24) & 8.97(0.23) & 9.77(0.29) & 11.31(0.21) & 11.62(0.56) \\
		& 50 & 9.0(0.22) & 9.27(0.26) & 9.15(0.15) & 9.38(0.18) & 10.61(0.03)\\
		\hline
		\multirowcell{2}{CIFAR-10}      
        & 10 & 13.54(0.23) & 13.7(0.25) & 14.43(0.22) & 16.82(0.29) & 19.7(0.54) \\
		& 50 & 13.55(0.18) & 13.75(0.09) & 14.19(0.22) & 15.69(0.21) & 18.84(0.26)\\
        \hline
        \end{tabular}
        \label{robust_table}
    \end{table}

\end{document}